\documentclass[11pt]{article}
\pdfoutput=1
\usepackage{hyperref}
\usepackage{amsmath}
\usepackage{amssymb}
\usepackage{amsthm}
\usepackage{amscd}
\usepackage{amsfonts}
\usepackage{graphicx}
\usepackage{fancyhdr}
\usepackage{color}
\usepackage{amsfonts}
\usepackage{amsthm}
\usepackage{psfrag} 
\usepackage{epsfig}
\usepackage{subfigure}
\usepackage{graphicx}
\usepackage{amssymb}
\usepackage{epstopdf}
\usepackage{amsbsy}
\usepackage{latexsym}
\usepackage{moreverb}    
\usepackage{textcomp}
\usepackage{bm}
\usepackage{soul}
\usepackage{mathtools}
\usepackage{thmtools}
\usepackage{thm-restate}
\usepackage{caption}
\usepackage{algpseudocode}[compatible]
\usepackage{algorithm}
\usepackage{enumerate}

\usepackage{multicol}

\declaretheorem{theorem}

\newtheorem{remark}{Remark}[section]

\newtheorem{corollary}{Corollary}[theorem]

\usepackage[textsize=tiny,textwidth=20mm]{todonotes}
\usepackage[top=2.8cm,bottom=2.8cm,left=2.5cm,right=2.5cm]{geometry}
\usepackage[T1]{fontenc}
\usepackage[utf8]{inputenc}

\usepackage[affil-it]{authblk} 
\usepackage{etoolbox}
\usepackage{lmodern}

\makeatletter
\patchcmd{\@maketitle}{\LARGE \@title}{\fontsize{18}{19.2}\selectfont\@title}{}{}
\makeatother

\usepackage{sectsty}
\sectionfont{\fontsize{13}{15}\selectfont}
\subsectionfont{\fontsize{12}{15}\selectfont}

\usepackage{enumerate}

\usepackage[textsize=tiny,textwidth=20mm]{todonotes}
\begin{document}

\title{Deep Learning without Global Optimization \\by Random Fourier Neural Networks}

\author[2]{Owen Davis\thanks{corresponding author: ondavis@sandia.gov}}
\author[2]{Gianluca Geraci\thanks{ggeraci@sandia.gov}}
\author[1]{Mohammad Motamed\thanks{motamed@unm.edu}}

\affil[1]{Department of Mathematics and Statistics, The University of New Mexico, Albuquerque, NM, USA}
\affil[2]{Sandia National Laboratories Department of Optimization and Uncertainty Quantification, Albuquerque, NM, USA}

\date{}

\maketitle

\begin{abstract}
We introduce a new training algorithm for deep neural networks that utilize random complex exponential activation functions. Our approach employs a Markov Chain Monte Carlo sampling procedure to iteratively train network layers, avoiding global and gradient-based optimization while maintaining error control. It consistently attains the theoretical approximation rate for residual networks with complex exponential activation functions, determined by network complexity. Additionally, it enables efficient learning of multiscale and high-frequency features, producing interpretable parameter distributions. Despite using sinusoidal basis functions, we do not observe Gibbs phenomena in approximating discontinuous target functions.
\end{abstract}

\medskip
\noindent
{\bf keywords:}
Deep neural networks, Markov Chain Monte Carlo sampling, random Fourier neural networks, sampling-based network training

\medskip
\noindent
{\bf MSCcodes:}
65T40, 90C15, 65C05, 65C40, 60J22, 68T07

\section{Introduction}

For several decades, global gradient descent-based optimization algorithms have been commonly used for deep neural network training primarily due to their empirical success in solving the very high-dimensional non-convex optimization problems that network training requires \cite{OND:goodfellow2016deep}. Despite this, these algorithms are not without drawbacks. Global gradient-based optimization can be very expensive for deep neural networks, the progression of training as well as the learned parameters are often difficult to interpret \cite{zhang2021survey}, and the performance of the algorithms regularly depends on several pre-training hyperparameters for which informed a priori choices are unavailable.

In addition to the above, it has been shown that deep neural networks, especially those using global gradient-based training algorithms, struggle to learn high-frequency and multiscale features of target functions with reasonable computational complexity, a phenomenon that is referred to as {\it spectral bias}; see e.g., \cite{rahaman2019spectral,basri2020frequency,xu2019frequency}. There has been considerable work towards understanding and mitigating spectral bias in both standard and physics informed neural networks; see e.g., \cite{cao2019towards,hong2022activation,wang2021understanding,wang2022and,tancik2020fourier,howard2023stacked,kammonen2024comparing}. In \cite{cao2019towards, wang2022and}, the authors establish a connection between spectral bias and the neural tangent kernel (NTK) \cite{jacot2018neural}, which describes the limiting behavior of the gradient-based training. The latter work further develops an adaptive procedure for weighing terms in the physics informed loss function by leveraging eigenvalues of the NTK. In \cite{kammonen2024comparing}, it is shown that an adaptively randomized training algorithm can alleviate spectral bias in one hidden layer neural networks as compared to training the same network with a global gradient-based optimizer. In \cite{hong2022activation}, it is shown that spectral bias is also related to the activation function of the network, and in particular that choosing a hat function activation over a ReLU activation can alleviate spectral bias. In \cite{howard2023stacked}, it is shown that training a sequence of deep neural networks in which the learned parameters in one network are used to initialize the parameters in the next can facilitate the learning of multiscale target function features and improve physics informed training. Perhaps most relevant to this work, the authors in \cite{tancik2020fourier} show that leveraging random Fourier features as a positional encoding strategy facilitates the learning of high frequency target function features.

In this work, we aim to address training issues related to spectral bias, completely sidestepping global and gradient-based optimization, while also enhancing interpretability in network training and the learned parameters. 
To accomplish this, we develop a global optimization-free training algorithm with error control for deep residual networks that utilize randomized complex exponential activation functions. Such activation functions are also known as random Fourier features \cite{rahimi2007random}. 
These networks, which we call ``random Fourier Neural Networks'' (rFNNs), were introduced in \cite{deepFF_kammonen}, following inspiration from their shallow counterparts in \cite{rahimi2008random,1layerFF_kammonen}. They exhibit similar approximation properties to ReLU networks \cite{Davis_Motamed:2024} and have the capability of effectively capturing high-frequency and multiscale features without excessive network complexity \cite{tancik2020fourier}. 
Our training algorithm employs a Markov Chain Monte Carlo (MCMC) sampling procedure to iteratively train network segments by sampling random frequencies associated with each of the complex exponential activation functions from an optimally derived distribution. 
We leverage the developed algorithm to learn a variety of target functions that show off various aspects of its behavior: 1) Across all test cases we achieve, and in several instances outperform, the only existing theoretical approximation rate for this network type \cite{deepFF_kammonen}; 2) We simultaneously capture both high- and low-frequency features of varying scales with minimal network complexity; 3) Our learned network parameters offer an interpretable frequency decomposition of the target function; and 4) Remarkably, despite utilizing a sinusoidal approximation basis, we do not observe Gibbs phenomena \cite{grafakos2008classical} in approximating discontinuous target functions.

\medskip

\noindent{\bf Relation to other work.} This work develops a global optimization-free network training algorithm with error control, realizing, and in several instances outperforming, the theoretical approximation rates for deep rFNNs derived in \cite{deepFF_kammonen}. There are several existing works that consider the problem of training one hidden layer neural networks using random features; see e.g. \cite{rahimi2008random,1layerFF_kammonen, li2019learning, gonon2023random, chen2022bridging, chen2023random}. In \cite{rahimi2008random,gonon2023random, chen2022bridging, chen2023random}, the network weights in a one hidden layer network are sampled randomly and only the output weights (coefficients of the random features) are trainable. This is done for purely data-driven networks in \cite{rahimi2008random,gonon2023random} and in a physics informed context for both stationary and time dependent partial differential equations in \cite{chen2022bridging,chen2023random}. 
In \cite{1layerFF_kammonen, li2019learning}, the authors instead seek to learn optimal random features from available training data. In \cite{li2019learning}, this is done via a gradient-based algorithm in a kernel regression setting, and in \cite{1layerFF_kammonen}, this is accomplished via a Metropolis algorithm.

The authors of  \cite{deepFF_kammonen} utilize the Metropolis algorithm from \cite{1layerFF_kammonen} to sequentially train segments of a deep random Fourier neural network by optimally sampling a selected subset of its frequency parameters. However, when used alone, this Metropolis algorithm is not demonstrated to achieve the theoretical approximation rate derived in \cite{deepFF_kammonen}. To achieve this rate, the authors employ the Metropolis procedure as a network initialization, subsequently applying a standard stochastic gradient descent-based method for global optimization of all network parameters. Inspired by the findings in \cite{1layerFF_kammonen, deepFF_kammonen}, we introduce an iterative MCMC technique that enables training of deep random Fourier neural networks without the need for subsequent global optimization, and meeting or exceeding the theoretical approximation rate. In Section 4.2, we conduct a direct comparison between the Metropolis approach in \cite{deepFF_kammonen} and our new training algorithm, revealing faster approximation rates and improved capabilities for capturing sharp and discontinuous features.

Our iterative approach to network training also has some high level similarities to both stacking networks \cite{howard2023stacked} and Galerkin networks \cite{ainsworth2021galerkin}. Both of these frameworks conduct iterative training of a sequence of neural networks and, therefore, come with a degree of error control, much like our proposed training algorithm. However, our work is different in several key ways. 

We sequentially train segments of a unified deep neural network using an MCMC sampling method, contrasting with Galerkin and stacking network paradigms that utilize gradient-based optimizers to train successive neural networks. Moreover, our training procedure yields interpretable frequency decompositions of the target function, a feature lacking in stacking and Galerkin networks. This work contributes to the growing literature at the intersection of Fourier analysis and deep learning; see e.g., \cite{sitzmann2020implicit,li2020fourier,li2021learnable}.
Each of these works successfully leverage Fourier representations of the target function to modify and greatly improve conventional global gradient-based network training for standard neural networks and neural fields \cite{sitzmann2020implicit}, operator learning frameworks \cite{li2020fourier}, and transformers \cite{li2021learnable}. Our work is similar to these works in the sense that we too leverage Fourier analysis in the context of deep learning, but our motivation is quite different; instead of leveraging Fourier analysis in conjunction with standard neural network training methodologies, we use it to sidestep conventional  global gradient-based training entirely. 

The rest of this work proceeds as follows. In Section~\ref{sec:random_FF_networks}, we define random Fourier neural networks and present their theoretical approximation rate. Section~\ref{sec:training_algorithm_design} offers a detailed exposition of our proposed training algorithm, concurrently developing the theoretical background that supports it. Section \ref{sec:numerics} showcases the developed training algorithm on a variety of numerical examples. Finally, concluding remarks and directions for future work are provided in Section \ref{sec:conclusion}.

\section{Random Fourier Neural Networks}\label{sec:random_FF_networks}

In this section, we define rFNNs and the space of target functions they can effectively model. We also present their existing generalization error estimate, which will aid in assessing the convergence rate of our proposed training algorithm.

\subsection{Target function space}

We consider approximating functions belonging to
\begin{equation}\label{target_function_space}
    S=\{Q:{\mathbb R}^d\to\mathbb{R}: ||Q||_{L^{1}(\mathbb{R}^d)}<\infty, \, ||\hat{Q}||_{L^{1}(\mathbb{R}^d)} < \infty\},
\end{equation}
where $\hat{Q}$ is the Fourier transform of $Q$. This function space can be succinctly described all absolutely integrable functions on $\mathbb{R}^d$ with absolutely integrable Fourier transform. These functions need only be continuous almost everywhere, so most common discontinuous functions appearing in science and engineering tasks are included.

\subsection{Definition of rFNNs}

Following \cite{deepFF_kammonen} we define the Fourier features activation function $s:\mathbb{R}\to \mathbb{C}$ by 
\begin{equation}\label{eqn:FF_activation}
    s(x) = e^{ix}, \qquad x\in \mathbb{R}.
\end{equation} 

An rFNN $\Phi$ having depth $L\geq 1$, width $W\geq 1$, and approximating a function $Q\in S$ is a network consisting of $L$ blocks, where the first block has one hidden layer with $W$ neurons, and the remaining blocks have one hidden layer with $2W$ neurons. The network realizes the function
\begin{equation}\label{eqn:rFNN}
    Q_{\Phi}(\theta) = z_{L}(\theta), \qquad \theta\in \mathbb{R}^d,
\end{equation}
where $z_{L}$ results from the recursive scheme
\begin{align*}
    z_{1}(\theta)&=\underbrace{\Re\sum_{j=1}^{W}b_{1,j}s(\omega_{1j} \cdot \theta)}_{g_{1}(\theta;\: \bm{\omega}_{1},\bm{b}_{1})};\\
    z_{\ell}(\theta) &= z_{\ell-1}(\theta) +\underbrace{\Re\sum_{j=1}^{W}b_{\ell,j}s(\omega_{\ell j}\cdot \theta)}_{g_{\ell}(\theta;\:\bm{\omega}_{\ell},\bm{b}_{\ell})}+\underbrace{\Re\sum_{j=1}^{W}b_{\ell,j}'s(\omega_{\ell j}'\cdot z_{\ell-1}(\theta))}_{g'_{\ell}(z_{\ell-1};\: \bm{\omega}'_{\ell},\bm{b}'_{\ell})}, \qquad \ell = 2,\dotsc, L,
\end{align*}
and where $\omega_{\ell j} \in \mathbb{R}^{d}$, $\omega_{\ell j}'\in\mathbb{R}$,  and $b_{\ell j},b_{\ell j}'\in\mathbb{C}$ are respectively frequency and amplitude parameters. For each block $\ell$, the two sets of frequency parameters
$$
\bm{\omega}_{\ell} := \{\omega_{\ell j}; \, j=1, \dotsc, W \}, \qquad  \bm{\omega}'_{\ell} := \{\omega_{\ell j}'; \, j=1, \dotsc, W \},
$$ 
are assumed to be independently and identically distributed (i.i.d.) random variables following two specific distributions $p_{\ell}(\omega): {\mathbb R}^d \rightarrow [0, \infty)$ and $q_{\ell}(\omega'): {\mathbb R} \rightarrow [0, \infty)$, respectively. 
That is,
\begin{equation}\label{eqn:random_freq}
\omega_{\ell j} \overset{\mathrm{iid}}{\sim}  p_{\ell}(\omega), 
\qquad
\omega_{\ell j}' \overset{\mathrm{iid}}{\sim}  q_{\ell}(\omega'), \qquad j=1, \dotsc, W,
\end{equation}
where $\omega$ and $\omega'$ are respectively arbitrary inputs to the functions $p_{\ell}$ and $q_{\ell}$. We can equivalently represent an rFNN by a sequence of frequency-amplitude tuples
$$
\Phi := \{ 
(\bm{\omega}_1,\bm{b}_1), \dotsc, (\bm{\omega}_{L},\bm{b}_{L}), (\bm{\omega}'_2,\bm{b}_2'),\dotsc,  (\bm{\omega}_{L}',\bm{b}_{L}') 
\},$$
$$\bm{\omega}_{\ell} \in {\mathbb R}^{Wd}, 
\quad
\bm{\omega}_{\ell}' \in {\mathbb R}^{W}, 
\quad
\bm{b}_{\ell}, \bm{b}_{\ell}' \in {\mathbb C}^{W}.
$$ 
In Figure~\ref{fig:FF_net_example}, we include an example diagram of an rFNN taking one-dimensional input with depth $3$ (or 3 blocks) and width $2$.
\begin{figure}[!htb]
    \centering
    \includegraphics[width = 0.75\textwidth]{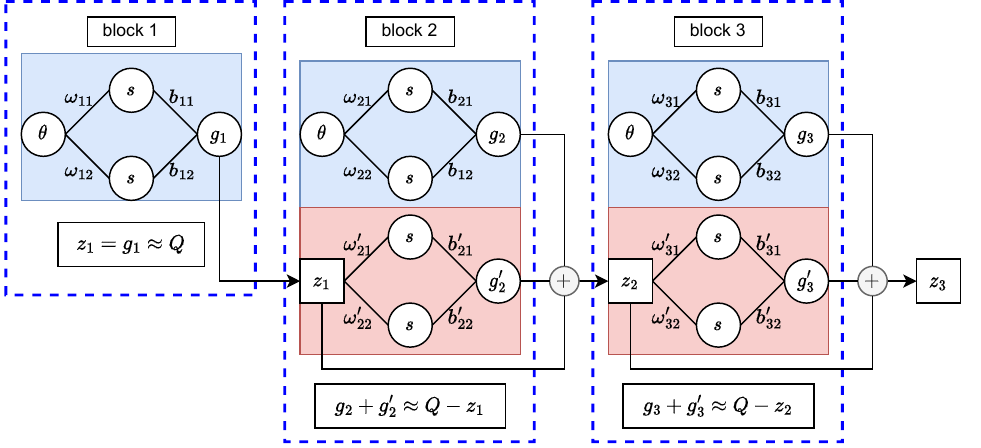}
    \caption{An rFNN with one input and $(W,L) = (2,3)$.}
    \label{fig:FF_net_example}
\end{figure}

An rFNN of depth $L$ proceeds as a series of $L$ blocks where each block learns a correction on the output of the previous block. Precisely, the first block, whose output is $z_{1}$, is a $W$ term Fourier sum approximation of the target function $Q$. Subsequently, at the $\ell$th block, with $\ell \ge 2$, the parameters $\bm{\omega}_{\ell}, \bm{\omega}_{\ell}', \bm{b}_{\ell}, \bm{b}_{\ell}'$ are tuned to approximate 
\begin{equation*}
    Q(\theta) - z_{\ell-1}(\theta)\approx g_{\ell}(\theta;\: \bm{\omega}_{\ell},\bm{b}_{\ell}) + g'_{\ell}(z_{\ell-1};\:\bm{\omega}'_{\ell},\bm{b}'_{\ell}), \qquad \ell \ge 2.
\end{equation*}
That is, in accordance with the form of the correction $Q-z_{\ell-1}$, we assume that it can be efficiently approximated by the sum of $g_{\ell}(\theta)$, which is just a function of $\theta$, and $g_{\ell}'(z_{\ell -1})$, which is just a function of the output of the previous block $z_{\ell - 1}$.  The output of block $\ell$ is then given by $z_{\ell} = z_{\ell-1} + g_{\ell} + g_{\ell}'$, where the output of the previous block $z_{\ell-1}$ is added to the approximation of the correction $g_{\ell}+g_{\ell}'$ through a skip-connection in the network architecture. Importantly, this makes the output of each block $\ell$ a true approximation of the target function; i.e. $z_{\ell}\approx Q$.

\subsection{Generalization error in rFNNs}

Assume that we are interested in using an rFNN $\Phi$ of depth $L$ and width $W$ to approximate a target function $Q\in S$ from a set of $N$ training data $\{(\theta^{(n)},Q(\theta^{(n)}))\}_{n=1}^{N}$, where $\{\theta^{(n)}\}_{n=1}^{N}$ are assumed to be i.i.d. samples from a (possibly unknown) distribution $\rho: \Theta  \rightarrow [0, \infty)$. 
Let ${\mathbb E}_{\theta}$ and $\mathbb{E}_{\bm{\omega},\bm{\omega}'}$ denote the expectation with respect to the input $\theta$ and the frequency parameters \{$\bm{\omega}_{\ell},\bm{\omega}_{\ell}'$; $\ell=1,\dotsc, L$\}, and denote by $\bm{b}$ and $\bm{b}'$ the collection of amplitude parameters $\{\bm{b_{\ell}}; \ell=1,\dotsc, L\}$ and $\{\bm{b}_{\ell}'; \ell=1,\dotsc, L\}$ respectively. We then define the generalization error $\varepsilon$ in the approximation of $Q$ by the rFNN as
\begin{equation}\label{eqn:generalization_error}
    \varepsilon := \mathbb{E}_{\bm{\omega},\bm{\omega}'}[\min_{\bm{b}, \bm{b}'}\mathbb{E}_{\theta}[|Q(\theta)-Q_{\Phi}(\theta;\bm{b},\bm{b}' \bm{\omega},\bm{\omega}')|^2]].
\end{equation}
The optimization problem \eqref{eqn:generalization_error} is an instance of the well known random Fourier features problem, which first appeared in \cite{rahimi2007random} and later in \cite{rahimi2008random,rudi2017generalization,weinan2020comparative,1layerFF_kammonen,deepFF_kammonen}. Moreover, it is to be noted that in practice the generalization error is often estimated by approximating the expectation ${\mathbb E}_{\theta}$ through sample averaging of the available $N < \infty$ data points. 

In \cite{deepFF_kammonen}, an upper bound for this generalization error was derived, a result that we restate using our notation in Theorem~\ref{thm:ff_main_result}.

\begin{theorem}\label{thm:ff_main_result} 
Let $Q$ be a target function in $S$, as defined in \eqref{target_function_space}, excluding the identically zero function.
Let $Q_{\Phi}$ be a random Fourier neural network \eqref{eqn:rFNN} with depth $L\geq 2$, width $W\geq 1$, and parameters $\{\bm{\omega}, \bm{\omega}', \bm{b}, \bm{b}'\}$. Then there exists positive constants $C'$ and $c$ such that 

\begin{equation}\label{eqn:FF_error_bound_true}
    \varepsilon \leq C'\frac{||Q||^{2}_{L^{\infty}(\mathbb{R}^d)}}{WL}\left(1 + \ln \frac{||\hat{Q}||_{L^{1}(\mathbb{R}^d)}}{||Q||_{L^{\infty}(\mathbb{R}^d)}} \right)^2 + \mathcal{O}\left(\frac{1}{W^2} + \frac{1}{L^4} + Le^{-cW}\right),
\end{equation}
where $\hat{Q}$ is the Fourier transform of $Q$. Furthermore, for sufficiently large $WL$, with $W =\mathcal{O}(L^2)$, there exists a positive constant $C$ such that the generalization error \eqref{eqn:generalization_error} satisfies
\begin{equation}\label{eqn:FF_error_bound}
    \varepsilon  \leq C\frac{||Q||^{2}_{L^{\infty}(\mathbb{R}^d)}}{WL}\left(1 + \ln \frac{||\hat{Q}||_{L^{1}(\mathbb{R}^d)}}{||Q||_{L^{\infty}(\mathbb{R}^d)}} \right)^2.
\end{equation}
\end{theorem}

\begin{proof}
The proof follows from a manipulation of Theorem 2.1 and Remark 2.1 in \cite{deepFF_kammonen}. 
\end{proof}
Interestingly, the result in Theorem~\ref{thm:ff_main_result} leads directly to the following corollary concerning the existence of particular (deterministic) Fourier neural networks with fixed frequency and amplitude parameters and with squared $L^2$ error satisfying the same estimates \eqref{eqn:FF_error_bound} and \eqref{eqn:FF_error_bound_true}.  

\begin{corollary}\label{cor:particular_FFnet}
Let $Q$ be a target function in $S$, as defined in \eqref{target_function_space}, and define
\begin{equation*}
    \varepsilon_{opt} := \min_{\bm{\omega},\bm{\omega}',\bm{b}, \bm{b}'}\{\mathbb{E}_{\theta}[|Q(\theta) -Q_{\Phi}(\theta;\bm{b},\bm{b}',\bm{\omega}, \bm{\omega}')|^2]\}.
\end{equation*}
There exists a Fourier neural network with fixed frequency and amplitude parameters, say $(\bm{\omega}^*,\bm{\omega}^{'*},\bm{b}^{*}, \bm{b}^{'*})$, that satisfies
\begin{equation}\label{eqn:deterministic_vs_random_true} 
    \varepsilon_{opt} \leq C'\frac{||Q||^{2}_{L^{\infty}(\mathbb{R}^d)}}{WL}\left(1 + \ln \frac{||\hat{Q}||_{L^{1}(\mathbb{R}^d)}}{||Q||_{L^{\infty}(\mathbb{R}^d)}} \right)^2 + \mathcal{O}\left(\frac{1}{W^2} + \frac{1}{L^4} + Le^{-cW}\right).
\end{equation}
Furthermore, for sufficiently large $WL$, with $W=\mathcal{O}(L^2)$, $\varepsilon_{opt}$ satisfies
\begin{equation}\label{eqn:deterministic_vs_random} 
    \varepsilon_{opt} \leq C\frac{||Q||^{2}_{L^{\infty}(\mathbb{R}^d)}}{WL}\left(1 + \ln \frac{||\hat{Q}||_{L^{1}(\mathbb{R}^d)}}{||Q||_{L^{\infty}(\mathbb{R}^d)}} \right)^2.
\end{equation}
\end{corollary}

\begin{proof}
First, assume $Q$ is not identically zero. Utilizing the fact that a minimum is less than or equal to its corresponding mean, we calculate 
{\footnotesize
    \begin{equation}\label{eqn:corr_min_less_mean}
        \min_{\bm{\omega},\bm{\omega}',\bm{b}, \bm{b}'}\{\mathbb{E}_{\theta}[|Q(\theta) -Q_{\Phi}(\theta;\bm{b},\bm{b}',\bm{\omega}, \bm{\omega}')|^2]\}\leq \mathbb{E}_{\bm{\omega},\bm{\omega}'}[\min_{\bm{b},\bm{b}'}\mathbb{E}_{\theta}[|Q(\theta) -Q_{\Phi}(\theta;\bm{b},\bm{b}' \bm{\omega}, \bm{\omega}')|^2]].
    \end{equation}}\ignorespacesafterend
Letting $\bm{b}^*, \bm{b}^{'*}, \bm{\omega}^*,\bm{\omega}^{'*}$ be the minimizers of the left hand side of \eqref{eqn:corr_min_less_mean}, the desired estimates \eqref{eqn:deterministic_vs_random_true} and \eqref{eqn:deterministic_vs_random} follow from Theorem \ref{thm:ff_main_result}. If $Q\equiv 0$, then it can be represented exactly by an rFNN of any width $W\geq 1$ and depth $L\geq 2$ by taking all amplitude parameters equal to zero. Hence the desired estimate holds for this target function as well.
\end{proof}

The estimates \eqref{eqn:FF_error_bound} and \eqref{eqn:deterministic_vs_random} indicate expected linear convergence in the approximation error with respect to the product of network width and depth $(WL)$. Among other metrics, these will be useful in evaluating the performance of our proposed training algorithm.

\section{Training Algorithm Design}\label{sec:training_algorithm_design}

This section is dedicated to detailing our proposed training algorithm, accompanied by a simultaneous development of the theoretical groundwork that supports it. In Section~\ref{sec:bbb_training_overview}, we derive optimal frequency parameter distributions specific to each block of the network, and motivate how those optimal frequency distributions enable a block-by-block training approach. Subsequently, Section~\ref{sec:adaptive_MCMC} introduces an adaptive MCMC procedure, utilizing the optimal frequency distributions to sequentially train each block of the network. This section further provides details on practical implementation of the algorithm and offers insights into its requisite hyperparameters. 

\subsection{A block-by-block training approach}\label{sec:bbb_training_overview}

In the present work, rather than optimizing all network parameters simultaneously, we opt to train each block of the network in sequence. 
This is motivated by the unique structure of rFNNs, where the two different kinds of random frequency parameters at each block follow distinct distributions, denoted by $p_{\ell}(\omega)$ and $q_{\ell}(\omega')$ in \eqref{eqn:random_freq}.
To this end, we first derive analytic a priori optimal frequency distributions, denoted as $p_{\ell}^*$ and $q_{\ell}^*$, for each block. Subsequently, our training strategy involves sampling frequencies from the optimal distribution(s) and then solving a convex optimization problem for the corresponding amplitudes. 

In \cite{deepFF_kammonen}, a similar strategy is leveraged.  The Metropolis procedure from \cite{1layerFF_kammonen} is used to iteratively train each network block by approximately sampling $\bm{\omega}_{\ell}$ from the optimal distribution $p^{*}_{\ell}$. In contrast, the frequencies $\bm{\omega}_{\ell}'$ are sampled just once from a normal distribution and then remain unchanged for the remainder of training. This modeling choice is justified in \cite{deepFF_kammonen} based on the assumption that, under optimal conditions, the term $g_{\ell}'(z_{\ell-1};\bm{b}'_{\ell},\bm{\omega}'_{\ell})$ learns a scaled identity map. We hypothesize that the assumptions made in \cite{deepFF_kammonen} regarding $g'_{\ell}(z_{\ell-1};\bm{b}'_{\ell},\bm{\omega}'_{\ell})$ and $\bm{\omega}_{\ell}'$ may not be optimal. Specifically, rFNNs can utilize two types of basis functions: the term $g_{\ell}(\theta; \bm{b}_{\ell}, \bm{\omega}_{\ell})$ enables the use of standard Fourier modes, while $g'_{\ell}(z_{\ell-1}; \bm{b}_{\ell}',\bm{\omega}_{\ell}')$ incorporates basis functions that are compositions of Fourier modes. The assumption in \cite{deepFF_kammonen} that $g'_{\ell}(z_{\ell-1}; \bm{b}_{\ell}',\bm{\omega}_{\ell}') $ learns a scaled identity map suggests that compositional basis functions are only marginally useful, implying that rFNNs should almost exclusively rely on standard Fourier modes to approximate features of the target function. We believe this assumption restricts the potential of  rFNNs to exploit the compositional power of neural network depth, particularly in approximating functions that standard Fourier modes do not handle well, such as those with discontinuities.

Aiming to optimally sample both types of random frequencies, we derive new optimal distributions $p_{\ell}(\omega)$ and $  q_{\ell}^{*}(\omega')$ for each block. These distributions enable us to create a training algorithm that consistently meets or even surpasses the theoretically predicted approximation rate. In Section~\ref{sec:discont}, we compare the Metropolis algorithm from \cite{deepFF_kammonen} with our block-by-block training approach, providing support for our hypothesis.

In the following derivation, we assume that we are utilizing an rFNN of width $W$ and depth $L$ to approximate a target function $Q\in S$. The derivation of the optimal frequency distributions differs between block $\ell=1$ and block $\ell>1$, so we divide our exposition. 
\medskip

\noindent{\bf Block $\ell = 1$:} At block $1$, the derivation follows the theoretical work in \cite{1layerFF_kammonen}, which we include here for completeness. We begin by deriving the known upper bound on the block $1$ generalization error; see e.g., \cite{barron1993universal, jones1992simple},
\begin{equation}
    \mathbb{E}_{\bm{\omega}_{1}}[\min_{\bm{b}_{1}}\{\mathbb{E}_{\theta}[|Q(\theta) - g_{1}(\theta;\bm{\omega}_{1},\bm{b}_{1})|^2] + \lambda_{1}|\bm{b}_{1}|^2 \}]\leq \frac{1+\lambda_1}{W}\mathbb{E}_{\omega}\left[\frac{|\hat{Q}(\omega)|^2}{(2\pi)^dp_1^2(\omega)}\right],
\end{equation}
where $\lambda_1\geq 0$ is a Tikhonov regularization parameter. Then as shown in \cite{1layerFF_kammonen}, this upper bound is minimized by the optimal frequency distribution
\begin{equation}\label{optimal_dist_block1}
    p_{1}^*(\omega) = \frac{|\hat{Q}(\omega)|}{||\hat{Q}||_{L^{1}(\mathbb{R}^d)}}.
\end{equation} 
\medskip

\noindent{\bf Block $\ell>1$:} Recall that at any block $\ell>1$, the network parameters are tuned to approximate the residual function 
\begin{equation*}
    r_{\ell}(\theta,z_{\ell-1}) = Q(\theta) - z_{\ell-1}\approx g_{\ell}(\theta;\bm{\omega}_{\ell},\bm{b}_{\ell}) + g_{\ell}'(z_{\ell-1};\bm{\omega}_{\ell}', \bm{b}_{\ell}').
\end{equation*}
To derive the optimal frequency distributions $p_{\ell}^*(\omega)$ and $q_{\ell}^*(\omega')$ for this block, we follow a similar approach to that used for block $1$, with one additional key assumption: $g_{\ell}(\theta)\approx \bar{r}_{\ell}(\theta)$ and $g_{\ell}'(z_{\ell-1}) \approx \bar{r}'_{\ell}(z_{\ell-1})$ for some unknown $\bar{r}_{\ell}$ and $\bar{r}'_{\ell}$. Importantly, we do not necessarily expect that $\bar{r}_{\ell}(\theta)=Q(\theta)$ and $\bar{r}'_{\ell}(z_{\ell-1})=-z_{\ell-1}$.  This assumption is motivated by the idea that certain features of the target function $r_{\ell}$ are more efficiently represented with respect to the variable $\theta$ while others are more efficiently represented with respect to the variable $z_{\ell-1}$, and that an optimal split between $\bar{r}_{\ell}$ and $\bar{r}'_{\ell}$ needs to be learned by the network. Given this assumption, we show in Appendix \ref{appendix} the following upper bound on the block $\ell$ generalization error 
{\small
\begin{equation*}
\begin{gathered}
    \mathbb{E}_{\bm{\omega}_{\ell},\bm{\omega}_{\ell}'}[\min_{\bm{b}_{\ell},\bm{b}'_{\ell}}\{\mathbb{E}_{\theta}[|r_{\ell}(\theta, z_{\ell-1})-g_{\ell}(\theta)-g_{\ell}'(z_{\ell-1})|^2]+\lambda_{\ell}|\bm{b}_{\ell},\bm{b}_{\ell}'|^2 \}]\\
    \leq\frac{1+\lambda_{\ell}}{W}\left(\mathbb{E}_{\omega}\left[\frac{|\hat{\bar{r}}_{\ell}(\omega)|^2}{(2\pi)^dp_{\ell}^2(\omega)}\right] + \mathbb{E}_{\omega'}\left[\frac{|\hat{\bar{r}}'_{\ell}(\omega')|^2}{(2\pi)q_{\ell}^2(\omega')}\right]\right),
    \end{gathered}
\end{equation*}}\ignorespacesafterend
where $\lambda_{\ell}\geq 0$ is a Tikhonov regularization parameter. Then as shown in Theorem \ref{thm:minimizing_densities} of Appendix \ref{appendix} this upper bound in minimized for the following optimal distributions written in terms of the unknown functions $\bar{r}_{\ell}$ and $\bar{r}'_{\ell}$
\begin{equation}\label{optimal_distributions}
    p^{*}_{\ell}(\omega) = \frac{|\hat{\bar{r}}_{\ell}(\omega)|}{||\hat{\bar{r}}_{\ell}||_{L^{1}(\mathbb{R}^d)}},\qquad q^{*}_{\ell}(\omega') = \frac{|\hat{\bar{r}}'_{\ell}(\omega')|}{||\hat{\bar{r}}'_{\ell}||_{L^{1}(\mathbb{R})}}.
\end{equation}

Now, given these optimal frequency distributions at each block, we consider the practical problem of approximating the target function $Q\in S$ from $N <\infty$ training samples $\{(\theta^{(n)},Q(\theta^{(n)}))\}_{n=1}^{N}.$ We decompose the training into a series of supervised learning problems, one for each block, where we do explicit training data augmentation between blocks. At block $\ell=1$, the training data is given by $\{(\theta^{(n)}, Q(\theta^{(n)})\}_{n=1}^{N_{1}}$, where $\{\theta^{(n)}\}_{n=1}^{N_1}$ are $N_1\leq N$ i.i.d. samples from some (possibly unknown) distribution $\rho_{1}:\Theta\mapsto [0,\infty)$. Then at any block $\ell>1$, the training data is given by 
\begin{equation*}
    \{(\theta^{(n)}, z_{\ell-1}^{(n)}, r_{\ell}(\theta^{(n)},z_{\ell-1}^{(n)}))\}_{n=1}^{N_{\ell}},
\end{equation*}
where $\{ \theta^{(n)} \}_{n=1}^{N_{\ell}}$ are $N_{\ell}\leq N$ i.i.d. samples from some (possibly unknown) distribution $\rho_{\ell}:\Theta\mapsto [0,\infty)$, and 
$z_{\ell-1}^{(n)} = z_{\ell-1}(\theta^{(n)})$ and 
$r_{\ell}(\theta^{(n)},z_{\ell-1}^{(n)}) = Q(\theta^{(n)}) - z_{\ell-1}^{(n)}$.

For simplicity, in our numerical examples in Section \ref{sec:numerics}, we use the very same $N$ training sample inputs at each block, but this is not required. The developed algorithm is flexible and can be applied to sample inputs which are disjoint, overlapping, or drawn from distributions specific to each block. A cost-accuracy analysis with respect to these different potential training data configurations is an exciting future research direction.  

From here, training at block $\ell=1$ is accomplished by minimizing the empirical risk on the block $\ell=1$ training data
\begin{equation}\label{empirical_risk_block1}
\begin{gathered}
    \mathbb{E}_{\bm{\omega}_{1}}[\min_{\bm{b}_{1}}\{N_{1}^{-1}\sum_{n=1}^{N_{1}}|Q(\theta^{(n)}) - g_{1}(\theta^{(n)};\bm{\omega}_1,\bm{b}_1)|^2 + \lambda_{1}|\bm{b}_{1}|^2 \}],\\ \qquad \omega_{1 j} \overset{\mathrm{iid}}{\sim}  p^{*}_{1}(\omega), \qquad j=1,\dotsc, W,
    \end{gathered}
\end{equation}
where $\lambda_{1}\geq0$ is a Tikhonov regularization parameter, $|\bm{b}_1|$ is the Euclidean norm of $\bm{b_1}\in \mathbb{C}^{W}$, and the frequencies are distributed according to the block $\ell=1$ optimal distribution \eqref{optimal_dist_block1}. Subsequently, our solution strategy entails generating a set of $W$ independent frequency samples, say $\bm{\omega}_{1}$, from $p_{1}^{*}(\omega)$ and then solving the following convex (least squares) optimization problem with the generated sample $\bm{\omega}_{1}$ for the amplitudes $\bm{b}_1$,
\begin{equation}\label{amplitudes_block_1}
    \min_{\bm{b}_{1}}\{N_{1}^{-1}\sum_{n=1}^{N_{1}}|Q(\theta^{(n)}) - g_{1}(\theta^{(n)};\bm{\omega}_{1},\bm{b}_1)|^2 + \lambda_{1}|\bm{b}_{1}|^2 \}.
\end{equation}

At block $\ell>1$, we take a similar approach. Suppressing the arguments in $g_{\ell}(\theta,\bm{\omega}_{\ell},\bm{b}_{\ell})$ and $g'_{\ell}(z_{\ell-1};\bm{\omega}_{\ell}',\bm{b}_{\ell}')$, we aim to minimize the empirical risk on the block's training data
\begin{equation}\label{empirical_risk_block_ell}
\begin{gathered}
\mathbb{E}_{\bm{\omega}_{\ell},\bm{\omega}_{\ell}'}[\min_{\bm{b}_{\ell},\bm{b}_{\ell}'}\{N_{\ell}^{-1}\sum_{n=1}^{N_{\ell}}|r_{\ell}(\theta^{(n)},z_{\ell-1}^{(n)}) - g_{\ell}(\theta^{(n)})- g'_{\ell}(z_{\ell-1}^{(n)})|^2 + \lambda_{\ell}|\bm{b}_{\ell},\bm{b}_{\ell}'|^2\}],\\
    \omega_{\ell j} \overset{\mathrm{iid}}{\sim}  p^{*}_{\ell}(\omega), 
\qquad
\omega_{\ell j}' \overset{\mathrm{iid}}{\sim}  q^{*}_{\ell}(\omega'), \qquad j=1, \dotsc, W,
\end{gathered}
\end{equation}
where $\lambda_{\ell}\geq 0$ is a Tikhonov regularization parameter, $|\bm{b}_{\ell},\bm{b}_{\ell}'|$ is a joint Euclidean norm of $\bm{b}_{\ell}$ and $\bm{b}_{\ell}'$ on $\mathbb{C}^{2W}$, and the frequency parameters are distributed according to the optimal distributions \eqref{optimal_distributions}. We then use a similar solution strategy. We sample $W$ independent frequencies, say $\bm{\omega}_{\ell}$, from $p^*_{\ell}(\omega)$, $W$ independent frequencies, say $\bm{\omega}_{\ell}'$, from $q^*_{\ell}(\omega')$, and then solve the following convex (least squares) optimization problem for the amplitudes $\bm{b}_{\ell}$ and $\bm{b}_{\ell}'$,
\begin{equation}\label{amplitudes_block_ell}
    \min_{\bm{b}_{\ell},\bm{b}_{\ell}'}\{N_{\ell}^{-1}\sum_{n=1}^{N_{\ell}}|r_{\ell}(\theta^{(n)},z_{\ell-1}^{(n)}) - g_{\ell}(\theta^{(n)};\bm{\omega}_{\ell},\bm{b}_{\ell})- g'_{\ell}(z_{\ell-1}^{(n)};\bm{\omega}'_{\ell},
    \bm{b}'_{\ell})|^2 + \lambda_{\ell}|\bm{b}_{\ell},\bm{b}_{\ell}'|^2\}.
\end{equation}

\subsection{Training via adaptive Markov Chain Monte Carlo}\label{sec:adaptive_MCMC}

Leveraging the optimal frequency distributions \eqref{optimal_dist_block1} and \eqref{optimal_distributions}, this section proposes an MCMC based procedure to sequentially train each block of the network. At block $1$, we aim to solve the optimization problem~\eqref{empirical_risk_block1}, where the optimal frequency distribution $p_1^*$ \eqref{optimal_dist_block1} is known and depends on the Fourier transform of the target function. Similarly, at block $\ell>1$, our goal is to solve the optimization problem~\eqref{empirical_risk_block_ell}, where the optimal frequency distributions $q_{\ell}^{*},p_{\ell}^{*}$ \eqref{optimal_distributions} are established and depend on the Fourier transforms of $\bar{r}_{\ell}$ and $\bar{r}'_{\ell}$. 

At block $\ell = 1$, efficiently computing the Fourier transform of the target function is often challenging, and at block $\ell>1$, the functions $\bar{r}_{\ell}$ and $\bar{r}'_{\ell}$ are unknown. Hence strategies are required to approximately sample the optimal distributions and approximately determine $\bar{r}_{\ell}$ and $\bar{r}'_{\ell}$. To address this, we devise an adaptive MCMC sampling approach inspired by a similar algorithm for random Fourier feature regression \cite{1layerFF_kammonen}, which is what is achieved by block $1$ of a deep rFNN. 
Importantly, in training any block $\ell>1$, harnessing our newly introduced optimal distribution $q^*_{\ell}(\omega')$  and adaptively determining $\bar{r}_{\ell}$ and $\bar{r}'_{\ell}$ requires the development of a Metropolis within Gibbs procedure that is considerably different from the strategy employed in 
\cite{1layerFF_kammonen}. 

In this section, for clarity, we begin by detailing just one step of our proposed sampling procedure at both block $1$ and block $\ell>1$. Subsequently, the full block-by-block training procedure is presented as Algorithm~\ref{alg:block_by_block_training} later in this section. Given the structural disparities between block $1$ and block $\ell>1$, we again divide our exposition.
\medskip

\noindent{\bf Block $\ell=1$:} At block $1$, we aim to solve the optimization problem \eqref{empirical_risk_block1}. Given that block $1$ implements standard random Fourier features regression, we directly leverage the Metropolis algorithm in \cite{1layerFF_kammonen}, which we include here for completeness. 
\begin{enumerate}
    \item At the beginning of the Metropolis loop we have current frequencies \\$\omega_{11} ,\dotsc, \omega_{1 W}\in \mathbb{R}^d$ with corresponding amplitudes $b_{1 1},\dotsc, b_{1 W}\in \mathbb{C}$.
    \item Conduct update of $\bm{\omega}_{1}$ as follows:
    \begin{enumerate}
        \item Propose new frequencies $\bar{\omega}_{11},\dotsc \bar{\omega}_{1W}$ from a symmetric proposal distribution.
        \item Using the proposed frequencies $\bm{\bar{\omega}}_{1}$, solve the convex optimization problem \eqref{amplitudes_block_1} for the corresponding amplitudes $\bm{\bar{b}}_{1}$.
        \item For $j=1,\dotsc, W$, accept the frequencies $\bar{\omega}_{1j}$ with probability \\$\min\{1, |\bar{b}_{1 j}|^{\gamma}/|b_{1 j}|^{\gamma}\}$, where $\gamma>0$ is a Metropolis selection hyperparameter.
    \end{enumerate}
\end{enumerate}
The acceptance criterion $|\bar{b}_{1j}|^{\gamma}/|b_{1j}|^{\gamma}$ used in this Metropolis sampling algorithm was introduced in \cite{1layerFF_kammonen}. It can be motivated in an asymptotic sense as $W,\gamma\to \infty$. For clarity, we sketch this argument here. In \cite{1layerFF_kammonen}, it is shown that as $W\to \infty$, $|b_{1j}|\propto |\hat{Q}(\omega_{1j})|/p_{1}(\omega_{1j})$, and ideally, we want to sample from the optimal frequency distribution $p_{1}^{*}$, which satisfies the proportionality relationship $p_{1}^{*}(\omega_{1j})\propto |\hat{Q}(\omega_{1j})|.$ Accomplishing this goal directly is computationally prohibitive, so we relax this condition and instead aim to sample from an auxiliary distribution $p_{\gamma}(\omega)$ satisfying
\begin{equation}\label{eqn:proportionality}
    p_{\gamma}(\omega_{1j})\propto |\hat{Q}(\omega_{1j})|^{\gamma/\gamma+1},
\end{equation}
where importantly, as $\gamma\to \infty$, we recover the desired proportionality relationship $p_{\gamma}(\omega_{1j})\propto |\hat{Q}(\omega_{1j})|$.  Rearranging \eqref{eqn:proportionality}, we find $p_{\gamma}(\omega_{1j})\propto |\hat{Q}(\omega_{1j})|^{\gamma}/(p_{\gamma}(\omega_{1j}))^{\gamma}\propto |b_{1j}|^{\gamma}$.  Hence as $W,\gamma\to \infty$, the quantity $|b_{1j}|^{\gamma}$ becomes proportional to $|\hat{Q}(\omega_{1j})|$, which is proportional to the optimal distribution $p_1^*(\omega_{1j})$ \eqref{optimal_dist_block1}. This provides motivation that this acceptance criterion can work within the context of a Metropolis sampling framework. Perhaps more importantly, in \cite{1layerFF_kammonen}, this acceptance criterion was empirically demonstrated to be effective even when $W,\gamma\ll \infty$, and we find the same in our numerical examples in Section~\ref{sec:numerics}.

\medskip

\noindent{\bf Block $\ell >1$:} Here, our objective is to solve the nested optimization problem \eqref{empirical_risk_block_ell}. To achieve this, we devise a Metropolis within Gibbs procedure, where we perform alternating updates of the frequencies $\bm{\omega}_{\ell}$ and $\bm{\omega}'_{\ell}$. 
\begin{enumerate}
    \item At the beginning of the Gibbs loop we have current frequencies $\omega_{\ell 1} ,\dotsc \omega_{\ell  W}\in \mathbb{R}^d$ and $\omega_{\ell  1}',\dotsc, \omega_{\ell W}'\in \mathbb{R}$ with corresponding amplitudes $b_{\ell  1},\dotsc, b_{\ell  W}\in \mathbb{C}$ and $b'_{\ell 1},\dotsc, b'_{\ell  W}\in \mathbb{C}$
    \item Conduct $\bm{\omega}_{\ell}$ update as follows:

    \begin{enumerate}
        \item Propose new frequencies $\bar{\omega}_{\ell 1},\dotsc, \bar{\omega}_{\ell W}$ from a symmetric proposal distribution.

        \item Using the frequencies $\bar{\bm{\omega}}_{\ell}, \bm{\omega}_{\ell}$ compute corresponding amplitudes $\bar{\bm {b}}_{\ell},\bar{\bm{b}}'_{\ell}$ by solving the inner optimization problem in \eqref{amplitudes_block_ell}.

         \item For $j=1,\dotsc, W$, accept frequencies $\bar{\omega}_{\ell  j}$ with probability \\$\min\{1,|\bar{b}_{\ell j}|^{\gamma}/|b_{\ell,j}|^{\gamma}\}$,  where $\gamma>0$ is Metropolis selection hyperparameter.
    \end{enumerate}
    \item Conduct update of $\bm{\omega}_{\ell}'$ as follows:
    \begin{enumerate}
        \item  Propose new frequencies $\bar{\omega}'_{\ell  1},\dotsc, \bar{\omega}_{\ell   W}'$ from a symmetric proposal distribution.
        \item Using the frequencies $\bar{\bm{\omega}}_{\ell}',  \bm{\omega}_{\ell}$, compute corresponding amplitudes $\bar{\bm{b}}_{\ell}',\bar{\bm{b}}_{\ell}$ by solving the inner optimization problem in \eqref{amplitudes_block_ell}.
       \item For $j=1,\dotsc, W$, accept the frequencies $\bar{\omega}'_{\ell  j}$ with probability\\ $\min\{1,|\bar{b}'_{\ell j}|^{\gamma'}/|b'_{\ell j}|^{\gamma'}\}$, where $\gamma'>0$ is a Metropolis selection hyperparameter.
    \end{enumerate}
\end{enumerate}
Inspired by the effectiveness of the acceptance criterion in the block $\ell=1$ case, we employ similar criteria at block $\ell>1$; we use the ratio $|\bar{b}_{\ell j}|^{\gamma}/|b_{\ell j}|^{\gamma}$ for frequencies $\omega_{\ell j}$ and the ratio $|\bar{b}'_{\ell j}|^{\gamma'}/|b'_{\ell j}|^{\gamma'}$ for frequencies $\omega_{\ell j}'$. Crucially, however, the amplitudes in these acceptance criteria $\bm{b}_{\ell}$, $\bm{b}_{\ell}'$ are computed concurrently in steps 2(b) and 3(b). This simultaneous optimization introduces a form of competition between the two different types of frequencies, enabling us to adaptively determine which features of the target function are most effectively represented with respect to the variable $\theta$ (associated with frequencies $\bm{\omega}_{\ell}$) and which features are most effectively represented with respect to the variable $z_{\ell-1}$ (associated with frequencies $\bm{\omega}'_{\ell}$).

\begin{remark}\label{rem:greedy}
    This algorithm can alternatively be understood as taking a greedy-type approach, where the high probability behavior is to sample frequencies with the largest corresponding amplitudes.
\end{remark}

\noindent {\bf Connection to global optimization} The \emph{global optimization} problem we consider in this work is to minimize the generalization error over the entire network \eqref{eqn:generalization_error}, which amounts to solving the following optimization problem
\begin{equation}\label{eqn:gen_error_remark}
    \mathbb{E}_{\bm{\omega},\bm{\omega}'}[\min_{\bm{b},\bm{b}'}\mathbb{E}_{\theta}[|Q(\theta)-Q_{\Phi}(\theta;\bm{b},\bm{b}', \bm{\omega},\bm{\omega}')|^2]].
\end{equation}
Importantly, the block-by-block training algorithm does not attempt to solve \eqref{eqn:gen_error_remark}. Instead, it aims to sequentially solve
{\small
\begin{equation}
\begin{gathered}\label{eqn:block_ell_generalization_error_rem}
    \mathbb{E}_{\bm{\omega}_{1}}[\min_{\bm{b}_{1}}\{\mathbb{E}_{\theta}[|Q(\theta)-g_{1}(\theta)|^2 + \lambda_{1}|\bm{b}_{1}|^2 \}], \qquad \ell=1;\\
    \mathbb{E}_{\bm{\omega}_{\ell},\bm{\omega}_{\ell}'}[\min_{\bm{b}_{\ell},\bm{b}'_{\ell}}\{\mathbb{E}_{\theta}[|r_{\ell}(\theta, z_{\ell-1})-g_{\ell}(\theta)-g_{\ell}'(z_{\ell-1})|^2]+\lambda_{\ell}|\bm{b}_{\ell},\bm{b}_{\ell}'|^2 \}], \qquad \ell = 2,\dotsc, L,
\end{gathered}
\end{equation}}\ignorespacesafterend
where $r_{\ell}(\theta, z_{\ell-1})=Q(\theta) - z_{\ell-1}$, and $z_{\ell-1} = z_{\ell-1}(\theta; \bm{\omega}_{\ell-1},\bm{b}_{\ell-1})$. That is, the block-by-block algorithm works to minimize the generalization error at each block in sequence, where the target function for a given block is explicitly updated based on the prediction from the previous block. Notably, there is not an equivalence between solving the global optimization problem \eqref{eqn:gen_error_remark} and the sequence of optimization problems \eqref{eqn:block_ell_generalization_error_rem} local to each block. Because of this, the approximation rate $\mathcal{O}(1/WL)$ is not necessarily the expected approximation rate and the architectural constraint $W=\mathcal{O}(L^2)$ is not necessarily required when training with the block-by-block algorithm. Nevertheless, as the only existing approximation rate for rFNNs, it is a natural benchmark to asses the developed training algorithm.

Moreover, we expect the block-by-block algorithm to outperform the theoretical approximation rate for certain functions. This rate is based on sampling all frequency parameters of a random Fourier neural network from a probability density that minimizes an upper bound on the generalization error for the entire network \eqref{eqn:gen_error_remark}. The value of this probability density for a given frequency is roughly proportional to its amplitude in the target function’s Fourier series representation.

In contrast, when solving the sequence of optimization problems \eqref{eqn:block_ell_generalization_error_rem}, the frequencies for each block are sampled from an optimal density specific to that block, which is derived based on explicitly targeting the residual function $r_{\ell}$. A clear example of where this approach is more effective is with multiscale target functions that have frequencies with widely varying amplitudes. For such target functions, if the frequency parameters of an rFNN are sampled from the density that minimizes the upper bound on the global generalization error \eqref{eqn:gen_error_remark}, the dominant frequency is likely to be sampled many times before the smaller frequencies are considered, which is inefficient.

However, with the block-by-block algorithm, the dominant frequency is likely to be sampled in the first block, while subsequent blocks will target the smaller frequencies with high-probability. By keeping the width of these network blocks small, this method significantly reduces the inefficiencies related to repeatedly sampling large-scale frequencies. In Section~\ref{sec:multiscale}, we approximate such a multiscale target function and the results empirically support the preceeding discussion.
\medskip

\noindent{\bf Real-valued formulation.} In practice, we implement a real-valued version of the previously described algorithm. This requires the real-valued formulation of the convex optimization problems \eqref{amplitudes_block_1} and \eqref{amplitudes_block_ell} for the amplitudes, as well as clarification on the frequency acceptance criteria. We discuss each of these in turn.

Consider approximating a target function $Q\in S$ utilizing a random Fourier neural network $\Phi$. This network realizes the function
$Q_{\Phi}(\theta)=z_{L}(\theta)$, where the the recursive scheme resulting in $z_{L}$ can be written using only real variables as
\begin{align}
    z_{1}(\theta) &= g_1(\theta) = \sum_{j=1}^{W}\Re(b_{1 j})\cos(\omega_{1 j}\cdot \theta) - \Im(b_{1 j})\sin(\omega_{1 j}\cdot\theta),\\
    z_{\ell}(\theta) &= z_{\ell-1}(\theta) + g_{\ell}(\theta) + g'_{\ell}(z_{\ell-1}), \qquad \ell=2,\dotsc, L,
\end{align}
where we take the explicit real variable form of $g_{\ell}$ and $g_{\ell}'$ given by 
\begin{align}
    g_{\ell}(\theta;\bm{\omega}_{\ell},\bm{b}_{\ell}) &= \sum_{j=1}^{W}\Re(b_{\ell j})\cos(\omega_{\ell j}\cdot \theta) - \Im(b_{\ell j})\sin(\omega_{\ell j}\cdot\theta),\\
    g'_{\ell}(z_{\ell-1};\bm{\omega}_{\ell}',\bm{b}_{\ell})&= \sum_{j=1}^{W}\Re(b'_{\ell j})\cos(\omega'_{\ell j}z_{\ell-1}) - \Im(b'_{\ell j})\sin(\omega'_{\ell j}z_{\ell-1}).
\end{align}
Given this formulation, we can define real valued optimization problems for the amplitudes at each block. At block $\ell=1$, the amplitudes 
\begin{equation*}
    \Re(\bm{b}_1) = (\Re(b_{11}),\dotsc, \Re(b_{1W}))\in \mathbb{R}^{W},\qquad \Im(\bm{b}_{1}) = (\Im(b_{11}),\dotsc, \Im(b_{1W}))\in \mathbb{R}^W
\end{equation*}
are obtained by solving 
\begin{equation}\label{eqn:least_squares_block1}
    \min_{\Re(\bm{b}_{1}),\Im(\bm{b}_{1})}\bigg\{N_{1}^{-1}\left|A_{1}\begin{bmatrix}
        \Re(\bm{b}_{1})\\
        \Im(\bm{b}_{1})
    \end{bmatrix} - \bm{r}_{1}\right|^2 + \lambda_{1}|(\Re(\bm{b}_{1}),\Im(\bm{b}_{1})|^2\bigg\},
\end{equation}
where $\bm{r}_{1} = (Q(\theta^{(1)}),\dotsc,Q(\theta^{(N_{1})}))^{\top}\in \mathbb{R}^{N_{1}}$, and 
$A_{1}\in\mathbb{R}^{N_{1}\times2W}$ has the following structure,
\begin{equation}\label{eqn: coeff_matrix_A1}
\begin{gathered}
    A_1 = \begin{bmatrix}
        [\cos(\omega_{1 j}\cdot \theta^{(n)})] & [-\sin(\omega_{1 j}\cdot \theta^{(n)})] 
    \end{bmatrix},\\
    n=1, \dotsc, N_{1}, \qquad j=1, \dotsc W.
    \end{gathered}
\end{equation}
Here, each row of $A_1$ corresponds to a different data input sample $\theta^{(n)}$ and each column corresponds to a different frequency $\omega_{1 j}$ considering both cosine and sine contributions. 
Subsequently, at any block $\ell>1$, the amplitudes $\Re(\bm{b}_{\ell}),\Im(\bm{b}_{\ell}),\Re(\bm{b}_{\ell}'),\Im(\bm{b}_{\ell}')\in \mathbb{R}^{W}$ are determined by solving
\begin{equation}
\label{eqn:least_squares_block_ell}   \min_{\substack{\Re(\bm{b}_{\ell}),\Im(\bm{b}_{\ell}),\\\Re(\bm{b}_{\ell}'),\Im(\bm{b}_{\ell}')}}\bigg\{N_{\ell}^{-1}\left|A_{\ell}\begin{bmatrix}
        \Re(\bm{b}_{\ell})\\
        \Im(\bm{b}_{\ell})\\
        \Re(\bm{b}_{\ell}')\\
        \Im(\bm{b}_{\ell}')
    \end{bmatrix} - \bm{r}_{\ell}\right|^2 + \lambda_{\ell}|(\Re(\bm{b}_{\ell}),\Im(\bm{b}_{\ell}),\Re(\bm{b}_{\ell}'),\Im(\bm{b}_{\ell}'))|^2\bigg\},
\end{equation}
where $\bm{r}_{\ell} = (Q(\theta^{(1)})-z_{\ell-1}(\theta^{(1)}),\dotsc,Q(\theta^{(N_{\ell})})-z_{\ell-1}(\theta^{(N_{\ell})}))^{\top}\in \mathbb{R}^{N_{\ell}}$, and $A_{\ell}\in\mathbb{R}^{N_{\ell}\times4W}$ has the following structure,
\begin{equation}\label{eqn: coeff_matrix}
\begin{gathered}
    A_{\ell} = \begin{bmatrix}
        [\cos(\omega_{\ell j}\cdot \theta^{(n)})] & [-\sin(\omega_{\ell j}\cdot \theta^{(n)})] & [\cos(\omega_{\ell j}'z_{\ell-1}^{(n)})] &  [-\sin(\omega_{\ell j}'z_{\ell-1}^{(n)})]
    \end{bmatrix},\\
    n=1, \dotsc, N_{\ell}, \qquad j=1, \dotsc W, \qquad \ell > 1.
    \end{gathered}
\end{equation}
Both \eqref{eqn:least_squares_block1} and \eqref{eqn:least_squares_block_ell}, being quadratic with respect to the amplitude parameters, are convex optimization problems that can be solved by several different strategies depending on the ratios $N_{1}/2W$ at block $\ell=1$ and $N_{\ell}/4W$ at block $\ell > 1$. When $N_{1}\sim 2W$ ($N_{\ell}\sim 4W$), we can leverage singular value or QR decomposition \cite{trefethen2022numerical}, and when $N_{1}\gg 2W$ ($N_{\ell}\gg 4W)$, gradient-based methods can be employed \cite{shalev2014understanding}. 

With this real variable formulation, the frequency acceptance criteria at each block $\ell=1,\dotsc, L$ in our random sampling procedures are computed using the real and imaginary components of the amplitudes, utilizing the relations: 
\begin{equation*}
    |b_{\ell j}| = \sqrt{\Re(b_{\ell j})^2 + \Im(b_{\ell j})^2}, \qquad |b_{\ell j}'| = \sqrt{\Re(b_{\ell j}')^2 + \Im(b_{\ell j}')^2}, \qquad j=1, \dotsc, W.
\end{equation*}
We further note that in the present work we consider only real valued target functions, which can always be represented using only non-negative frequencies. Hence all negative valued frequencies are rejected during training.

Given this real-valued formulation, our full block-by-block training strategy is presented in Algorithm~\ref{alg:block_by_block_training}.
\begin{algorithm}[!htb]
    \caption{Block-by-block Training}\label{alg:block_by_block_training}
    \begin{algorithmic}[1]
    \State{{\bf Input:} training data: $\{(\theta^{(n)},Q(\theta^{(n)})\}_{n=1}^{N}$}
    \State{{\bf Output:} trained parameters: $\{\Re(\bm{b}_{\ell}),\Im(\bm{b}_{\ell}), \Re(
    \bm{b}'_{\ell}),\Im(\bm{b}'_{\ell}),\bm{\omega}_{\ell},\bm{\omega}_{\ell}'\}_{\ell=1}^{L}$}
    \State{{\bf Choose:}}
    \State{$W:= \text{network width}$}
    \State{$L:=\text{network depth}$}
    \State{$M:=\text{number of Metropolis iterations at each block}$}
    \State{$\gamma,\gamma':=\text{acceptance criteria exponents}$}
    \State{$\delta,\delta':=\text{Gaussian proposal variances}$}
    \State{$(\lambda_{1},\dotsc, \lambda_{L}) :=\text{Tikhonov regularization parameters associated with \eqref{eqn:least_squares_block1} and \eqref{eqn:least_squares_block_ell}}$}
    \algstore{myalg}
    \end{algorithmic}
    \end{algorithm}
    \begin{algorithm}
    \vspace{0.2cm}

    \begin{algorithmic}[1]
    \algrestore{myalg}
    \State{{\bf Train block 1:}}
    \State{$\bm{r} \leftarrow (Q(\theta^{(1)}),\dotsc,Q(\theta^{(n)}))$}
    \State{$\omega_{1j}\sim \mathcal{N}(0,I_d)$, \qquad $j=1,\dotsc, W$ \Comment{Initialize frequencies $\bm{\omega}_{1}$}}
    
    \State{$\Re(\bm{b}_1),\Im(\bm{b}_1) \leftarrow \text{minimizer of \eqref{eqn:least_squares_block1} given }\bm{\omega}_1, \bm{r}$}
    \For{$i=1,\dotsc, M$} \Comment{Begin Metropolis loop}
    \State{$\bar{\omega}_{1j}\sim \mathcal{N}(\omega_{1j},\delta I_{d})$,\qquad $j=1,\dotsc,W$} \Comment{Propose new frequencies $\bar{\bm{\omega}}_{1}$}
    \State{$\Re(\bar{\bm{b}}_{1}),\Im(\bar{\bm{b}}_{1})\leftarrow 
    \text{minimizer of \eqref{eqn:least_squares_block1} given } \bar{\bm{\omega}}_{1}, \bm{r}$}
    \For{$j=1,\dotsc, W$} \Comment{Begin $\bm{\omega}_1$ update loop}
    \If{$\left(\sqrt{\Re(\bar{b}_{1j})^2+\Im(\bar{b}_{1j})^2}/\sqrt{\Re(b_{1j})^2+\Im(b_{1j})^2}\right)^{\gamma}>\mathcal{U}(0,1)$}
    \State{$\omega_{1j}\leftarrow \bar{\omega}_{1j}, \quad \Re(b_{1j}),\Im(b_{1j}) \leftarrow \Re(\bar{b}_{1j}),\Im(\bar{b}_{1j})$}
    \EndIf
    \EndFor
    \State{$\Re(\bm{b}_1),\Im(\bm{b}_1)\leftarrow \text{minimizer of \eqref{eqn:least_squares_block1} given } \bm{\omega}_1, \bm{r}$}
    \EndFor
    \State{$\bm{z} = (z^{(1)},\dotsc, z^{(N)})\leftarrow (g_1(\theta^{(1)};\bm{\omega}_1,\Re(\bm{b}_1),\Im(\bm{b}_1)),\dotsc,g_1(\theta^{(N)};\omega_1,\Re(\bm{b}_1),\Im(\bm{b}_1))) $}
    \algstore{myalg}
    \end{algorithmic}
    \end{algorithm}
    \begin{algorithm}[!htb]

    \begin{algorithmic}[1]
    \algrestore{myalg}
    \State{{\bf Train blocks 2 through L:}}
    \For{$\ell = 2, \dotsc, L$:}
    \State{$\bm{r} \leftarrow \bm{r}-\bm{z}$} \Comment{Initialize frequencies $\bm{\omega}_{\ell}$}
    \State{$\omega_{\ell j}\sim\mathcal{N}(0,I_d)$, \qquad $j=1,\dotsc, W$}\Comment{Initialize frequencies $\bm{\omega}'_{\ell}$}
    \State{$\omega_{\ell j}'\sim\mathcal{N}(0,I_1)$, \qquad $j=1,\dotsc, W$}
    \State{$\Re(\bm{b}_{\ell}),\Im(\bm{b}_{\ell}),\Re(\bm{b}_{\ell}'),\Im(\bm{b}_{\ell}')\leftarrow \text{minimizer of \eqref{eqn:least_squares_block_ell} given } \bm{\omega}_{\ell},\bm{\omega}_{\ell}', \bm{r}$}
    \For{$i=1,\dotsc, M$:} \Comment{Begin Metropolis within Gibbs loop}
    \State{$\bar{\omega}_{\ell j}\sim \mathcal{N}(\omega_{\ell j},\delta I_d)$\qquad $j=1,\dotsc, W$} \Comment{Propose new frequencies $\bar{\bm{\omega}}_{\ell}$}
    \State{$\Re(\bar{\bm{b}}_{\ell}),\Im(\bar{\bm{b}}_{\ell}),\Re(\bar{\bm{b}}'_{\ell}),\Im(\bar{\bm{b}}'_{\ell})\leftarrow 
    \text{minimizer of \eqref{eqn:least_squares_block_ell} given } \bar{\bm{\omega}}_{\ell},\bm{\omega}_{\ell}', \bm{r}$}
    \For{$j=1,\dotsc, W$} \Comment{Begin 
    $\bm{\omega}_{\ell}$ update loop}
    \If{$\left(\sqrt{\Re(\bar{b}_{\ell j})^2+\Im(\bar{b}_{\ell j})^2}/\sqrt{\Re(b_{\ell j})^2+\Im(b_{\ell,j})^2}\right)^{\gamma}>\mathcal{U}(0,1)$}
    \State{$\omega_{\ell j}\leftarrow \bar{\omega}_{\ell j}, \quad \Re(b_{\ell j}),\Im(b_{\ell j}) \leftarrow \Re(\bar{b}_{\ell j}),\Im(\bar{b}_{\ell j})$}
    \EndIf
    \EndFor

    \State{$\bar{\omega}'_{\ell j}\sim \mathcal{N}(\omega'_{\ell j},\delta')$,\qquad $j=1,\dotsc W$}\Comment{Propose new frequencies $\bar{\bm{\omega}}'_{\ell}$}
    \State{$\Re(\bar{\bm{b}}_{\ell}),\Im(\bar{\bm{b}}_{\ell}),\Re(\bar{\bm{b}}'_{\ell}),\Im(\bar{\bm{b}}'_{\ell})\leftarrow 
    \text{minimizer of \eqref{eqn:least_squares_block_ell} given } \bm{\omega}_{\ell},\bar{\bm{\omega}}'_{\ell}, \bm{r}$}
    \For{$j=1,\dotsc, W$} \Comment{Begin $\bm{\omega}'_{\ell}$ update loop}
    \If{$\left(\sqrt{\Re(\bar{b}'_{\ell j})^2+\Im(\bar{b}'_{\ell j})^2}/\sqrt{\Re(b'_{\ell j})^2+\Im(b'_{\ell j})^2}\right)^{\gamma'}>\mathcal{U}(0,1)$}
    \State{$\omega'_{\ell j}\leftarrow \bar{\omega}'_{\ell j}, \quad \Re(b'_{\ell j}),\Im(b'_{\ell j}) \leftarrow \Re(\bar{b}'_{\ell j}),\Im(\bar{b}'_{\ell j})$}
    \EndIf
    \EndFor
    \State{$\Re(\bm{b}_{\ell}),\Im(\bm{b}_{\ell}),\Re(\bm{b}_{\ell}'),\Im(\bm{b}_{\ell}')\leftarrow \text{minimizer of \eqref{eqn:least_squares_block_ell} given } \bm{\omega}_{\ell},\bm{\omega}_{\ell}', \bm{r}$}
    \EndFor
    \State{$\bm{g}\leftarrow (g_{\ell}(\theta^{(1)};\bm{\omega}_{\ell};\Re(\bm{b}_{\ell}),\Im(\bm{b}_{\ell})),\dotsc, g_{\ell}(\theta^{(N)};\bm{\omega}_{\ell};\Re(\bm{b}_{\ell}),\Im(\bm{b}_{\ell})))$}
    \State{$\bm{g}'\leftarrow (g_{\ell}'(z^{(1)};\bm{\omega}_{\ell}';\Re(\bm{b}_{\ell}'),\Im(\bm{b}_{\ell}')),\dotsc, g_{\ell}'(z^{(N)};\bm{\omega}_{\ell}';\Re(\bm{b}_{\ell}'),\Im(\bm{b}_{\ell}')))$}
    \State{$\bm{z}\leftarrow \bm{z} + \bm{g} + \bm{g}'$}
    \EndFor
    \State{}
    \State{{\bf Return:} $\{\Re(\bm{b}_{\ell}),\Im(\bm{b}_{\ell}), \Re(
    \bm{b}'_{\ell}),\Im(\bm{b}'_{\ell}),\bm{\omega}_{\ell},\bm{\omega}_{\ell}'\}_{\ell=1}^{L}$}
    \end{algorithmic}
\end{algorithm}

\noindent {\bf Hyperparameter selection.}
We conclude this section by discussing the network architecture and hyperparameter choices present in the block-by-block training algorithm and outlined in Table~\ref{tab:hyperparameter_descriptions}.
\begin{table}[!htb]
    \centering
    \begin{tabular}{|c|p{10cm}|}
    \hline
         $W$ & Network width \\
         \hline
         $L$ & Network depth \\
         \hline
         $M$ & Number of Metropolis iterations (block $1$) and Metropolis within Gibbs iterations (block $\ell>1$) \\
         \hline
         $\gamma$, $\gamma'$ & Acceptance criteria exponents respectively associated with  frequencies $\omega_{\ell j}$, $\omega'_{\ell j}$ for all $j=1,\dotsc, W$, $\ell=1,\dotsc, L$ \\
         \hline
         $\delta, \delta'$ & Variance of Gaussian proposal distribution respectively associated with sampling frequency $\omega_{\ell,j}, \omega_{\ell,j}'$ for all $j=1,\dotsc,W$, $\ell=1,\dotsc, L$\\
         \hline
         $\lambda_{1},\dotsc, \lambda_{L}$ & Tikhonov regularization parameters at each block $\ell=1,\dotsc, L$\\ 
         \hline 
    \end{tabular}
    \caption{Block-by-block training  hyperparameters}
    \label{tab:hyperparameter_descriptions}
\end{table}

Regarding the choice of network architecture, we note that there are several reasons to keep network width small and increase network complexity primarily through depth. By maintaining a small $W$, we achieve a computationally efficient least squares problem for the amplitudes and potentially enhance our ability to train on sparser data. Indeed, for small $N_{\ell}$ ($\ell=1,\dotsc, L$), an even smaller $W$ is required to retain an overdetermined problem for the amplitudes at each block.

Furthermore, at each block, the width $W$ represents the number of Markov chains simultaneously attempting to sample the same optimal distribution. Hence periodically during training, different chains may sample similar frequencies simultaneously, leading to near-linear dependence in the least squares problems \eqref{eqn:least_squares_block1} and \eqref{eqn:least_squares_block_ell}. This concern is primarily handled through Tikhonov regularization, but if very small tolerances are desired, keeping width small is an effective and easy strategy to avoid ill conditioned least squares problems associated with this near linear dependence. We defer exploration into further methods to address this near linear dependence as future research.

It is additionally noteworthy that the existing theoretical approximation rate for random Fourier neural networks imposes a theoretical architecture constraint of $W=\mathcal{O}(L^2)$; see Theorem~\ref{thm:ff_main_result}. However, in practice, we have found that this constraint is not necessary. In fact, we observe the theoretical approximation rate in all of our numerical examples even when $W<L$. This empirical observation strengthens the case for block-by-block training, which inherently includes error control with respect to network depth. Since an approximation of the target function is obtained after each block, there is no need to choose $L$ before training. Blocks can be added incrementally until a desired error tolerance is achieved.

We denote by $M$ the number of Metropolis iterations at block $\ell=1$ and the number of Metropolis within Gibbs iterations at block $\ell>1$. This value can be predetermined and fixed, or it can be chosen adaptively using one of the many available MCMC stopping criteria; see e.g. \cite{roy2020convergence}. Therefore, it is not necessary for it to remain consistent across all blocks, and it should not be perceived as an inflexible pre-training hyperparameter.

The parameters $\delta$ and $\delta'$ represent the variance in the Gaussian proposal distributions used in our Metropolis (block $\ell = 1$) and Metropolis within Gibbs (block $\ell>1$) procedures to propose frequencies $\omega_{\ell,j}$ and $\omega_{\ell j}'$ respectively. We clarify here that given a current frequency parameter, say $\omega_{\ell j}\in \mathbb{R}^d$, we propose a new frequency $\bar{\omega}_{\ell j}$ sampled from a (multivariate) normal distribution, e.g.,   
$\bar{\omega}_{\ell j}\sim \mathcal{N}(\omega_{\ell j}, \delta I_d)$. 
Proposals concerning $\omega_{\ell j}'$ are accomplished in the same way with variance $\delta'$, but we note that $\omega_{\ell j}'$ is 1-dimensional for all $\ell=2,\dotsc, L$, $j=1,\dotsc, W$. In the present work, we exhibit fully satisfactory results on an array of target functions, but we remark that the use of the diagonal covariance matrix $\delta I_{d}$ for multidimensional frequencies is certainly not optimal and will be improved in future iterations of the algorithm. We further draw the reader's attention to the work \cite{roberts2001optimal}, where it was shown that the optimal variance in a general random walk Metropolis proposal distribution is given by $ 2.4^2/d$, where $d$ is the dimension of the target distribution. Certainly, for a particular Metropolis sampling procedure, more optimal procedures could be leveraged to optimally tune the hyperparameters $\delta,\delta'$; see, for example, \cite{haario2001adaptive}, which considers adaptive updating of the proposal variance. However, the choice $\delta,\delta' = 2.4^2/d$ represents a good initial guess, and in the present work, this choice of proposal variance yielded satisfactory results and desirable MCMC acceptance rates of about 30\% during training. Future iterations of the algorithm will consider adaptive updating of the proposal distributions; see e.g. \cite{haario2001adaptive}.

The hyperparameters $\lambda_{1},\dotsc, \lambda_{L}$ are utilized to regularize the least squares problems \eqref{eqn:least_squares_block1} and \eqref{eqn:least_squares_block_ell} throughout training. This serves to prevent ill conditioning resulting from linear dependence and also aids in mitigating overfitting. The selection of Tikhonov regularization parameter values largely depends on various factors such as the desired tolerance, number of training data, noise in the training data, network architecture, etc. A thorough investigation concerning how to optimally handle regularization for our proposed algorithm is deferred as a future research direction.

\section{Numerical Examples}\label{sec:numerics}

In this section, we approximate several target functions of varying regularity and dimension using rFNNs trained with our developed block-by-block algorithm. The $N$ training inputs $\{\theta^{(n)}\}_{n=1}^{N}$ are sampled from a random uniform distribution over the domain, and prior to training, the data are normalized with respect to a standard normal distribution. 
We use the very same $N$ training sample inputs in every block. In the notation of Section \ref{sec:bbb_training_overview}, this can be expressed as $N_{1} = \cdots = N_{L} = N$. Furthermore, for all numerical examples, we consider a network architecture and training data allotment such that the data-to-feature ratios are sufficiently large as to guarantee that the least squares problems \eqref{eqn:least_squares_block1} and \eqref{eqn:least_squares_block_ell} are well conditioned. 
We implement our algorithm in \verb!Julia! \cite{Julia-2017} and use the backslash operator \verb!\! to solve our regularized least squares problems \eqref{eqn:least_squares_block1} and \eqref{eqn:least_squares_block_ell} during training. For overdetermined systems, such as those considered in this work, this amounts to solving via QR decomposition, a method known to be highly numerically stable \cite{trefethen2022numerical}.

We evaluate the results both qualitatively and quantitatively using mean squared error
\begin{equation}\label{MSE}
    \varepsilon_{MSE} = \frac{1}{N_{test}}\sum_{n=1}^{N_{test}}|Q(\theta^{(n)})-Q_{\Phi}(\theta^{(n)})|^2,
\end{equation}
where $N_{test}$ is the number of samples in our test set. Our test set is always uniformly distributed over the prescribed domain. It should be noted that as $N_{test}\to \infty$, $\varepsilon_{MSE}\to \mathbb{E}_{\theta}[|Q(\theta)-Q_{\Phi}(\theta)|^2]$. Thus by Corollary~\ref{cor:particular_FFnet}, for sufficiently large $N_{test}$, $\varepsilon_{MSE}$ should exhibit the same approximation rate as the generalization error \eqref{eqn:generalization_error} in terms of network complexity. For clarity, we include experimental design choices, such as network settings, training data allocations, hyperparameter selections, etc., associated with each of our numerical experiments in Appendix \ref{app:experimental_design}.

\subsection{A multiscale target function}\label{sec:multiscale}

Consider the target function
\begin{equation*}
Q(\theta) = \cos(4 \, \theta) + 0.3\cos(70 \, \theta) + 0.05\cos(150 \, \theta),\qquad \theta\in [-1,1].
\end{equation*}
This function is relevant because it includes very high and very low frequencies of very different amplitudes. If fact, both the frequencies and amplitudes differ by more than an order of magnitude, and as such this function represents a very challenging learning task for a neural network suffering from spectral bias. In Figure~\ref{fig:multiscale_learning_progress}, we exhibit, from left to right, the training progress of an rFNN of width $W=6$ over the first three blocks. On the top row, we plot the network predictions, and on the bottom row, we provide corresponding histograms of the accepted frequencies at each block after a burn-in period of $2000$ iterations.

There are several important takeaways from these results. First, since the blocks are trained sequentially, and since we obtain a true approximation of the target function after each block, there is no need to choose the network depth ahead of time. We obtain updated error estimates after every block, and training can be terminated when a desired error tolerance is reached. Second, although this function is multiscale and contains high frequency features, we are able to learn all three frequencies with very low network complexity. Indeed, we have a concrete notion of the frequencies present in the target function by block 1. Moreover, the learned frequencies are the true frequencies of the target function. 
\begin{figure}[!htb]
    \centering
    \includegraphics[width=0.33\textwidth]{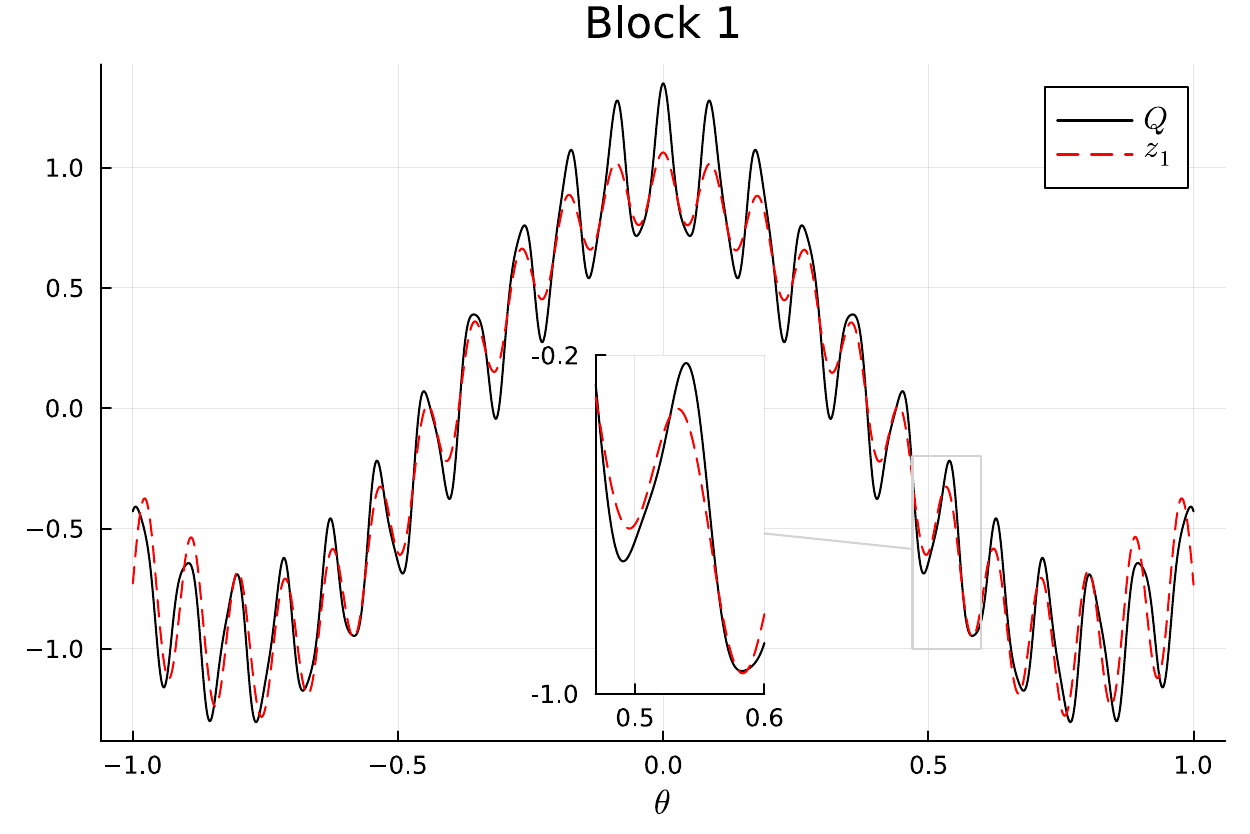}\includegraphics[width=0.33\textwidth]{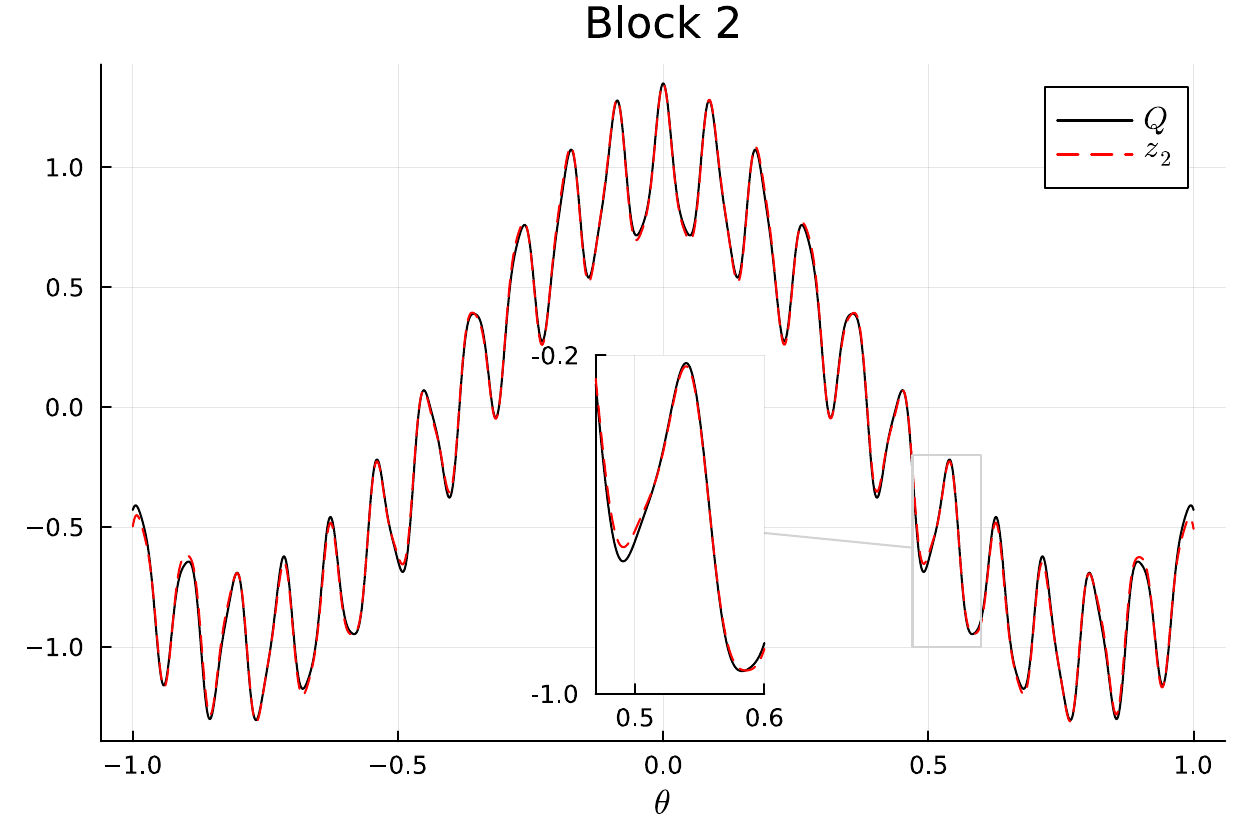}\includegraphics[width=0.33\textwidth]{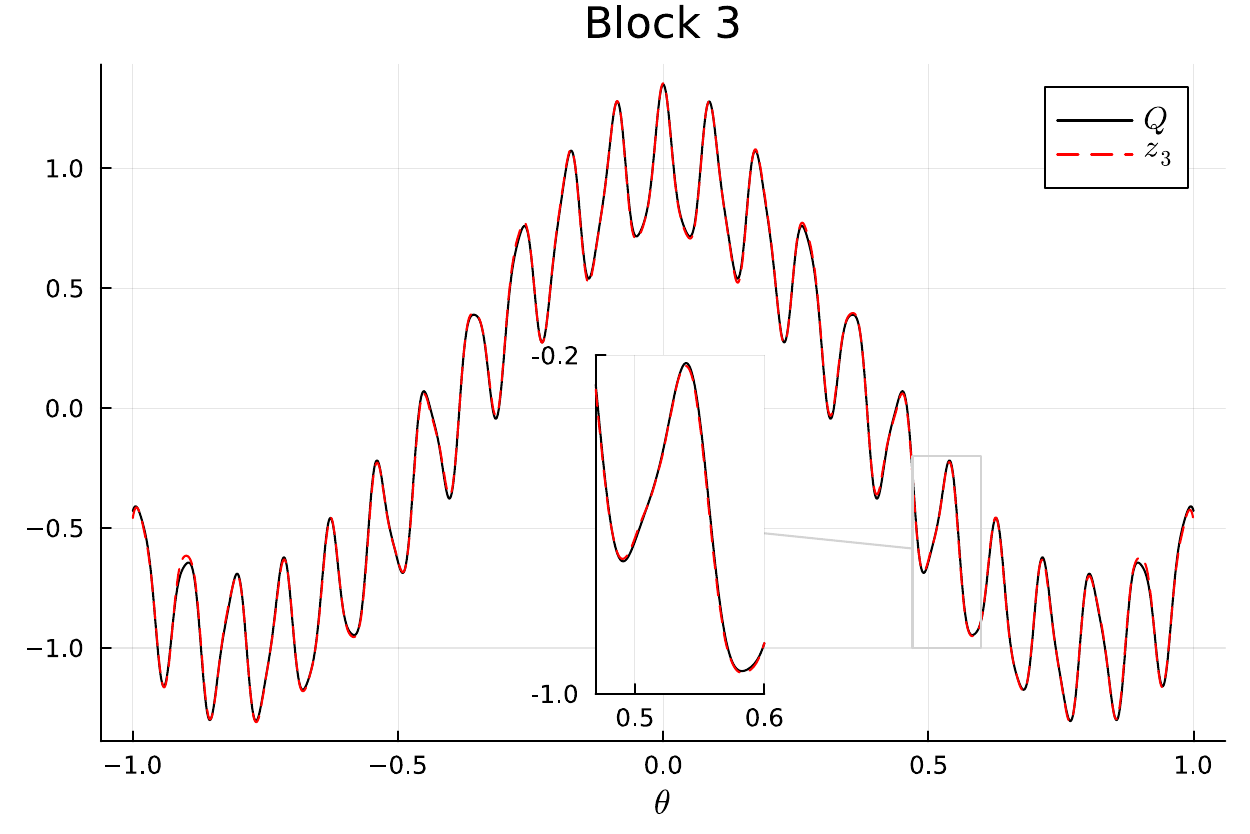}
    \includegraphics[width=0.33\textwidth]{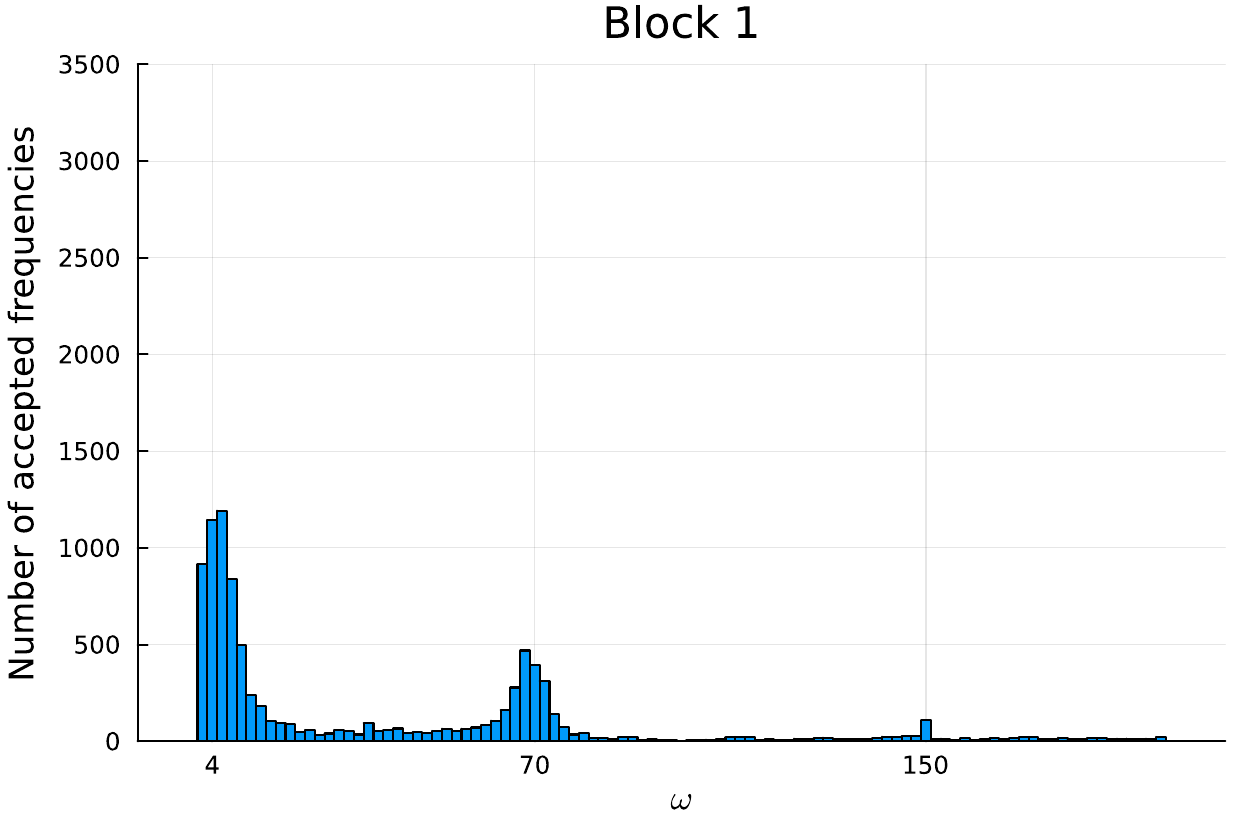}\includegraphics[width=0.33\textwidth]{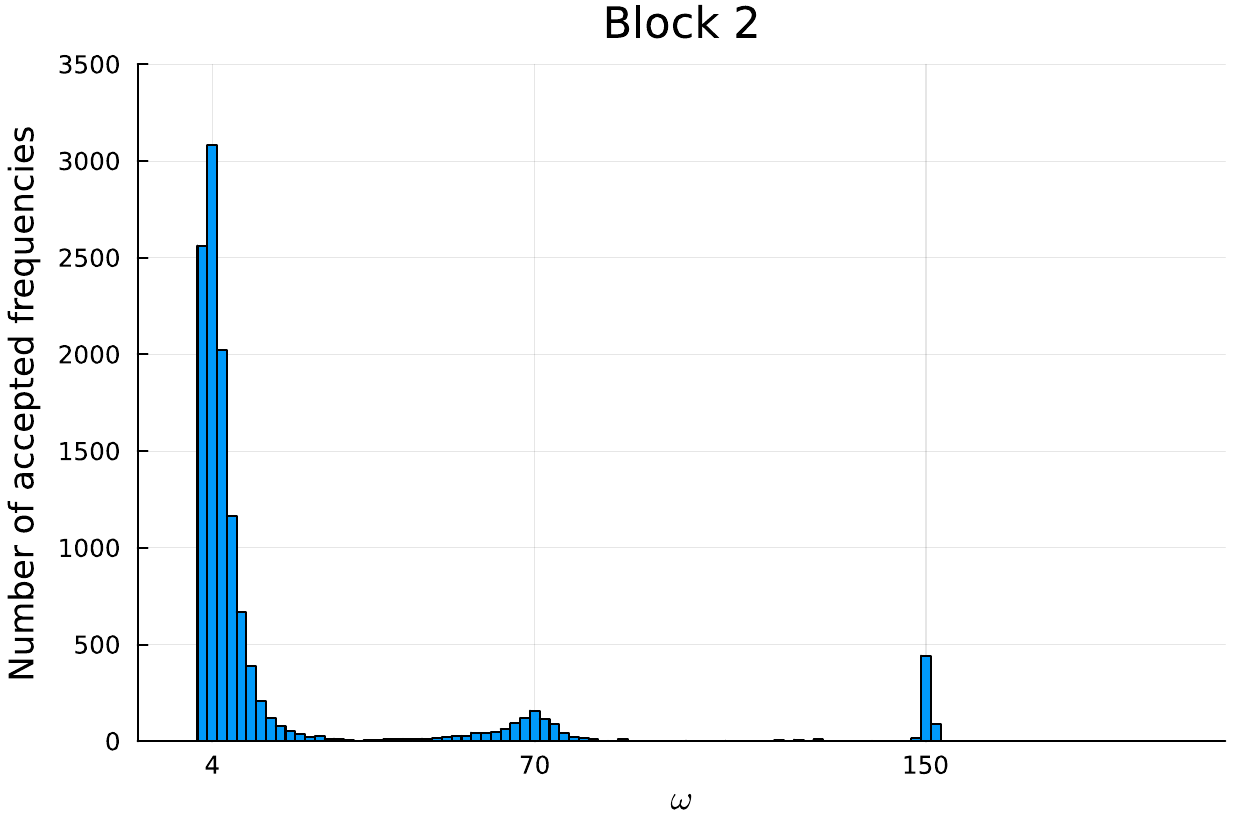}
    \includegraphics[width=0.33\textwidth]{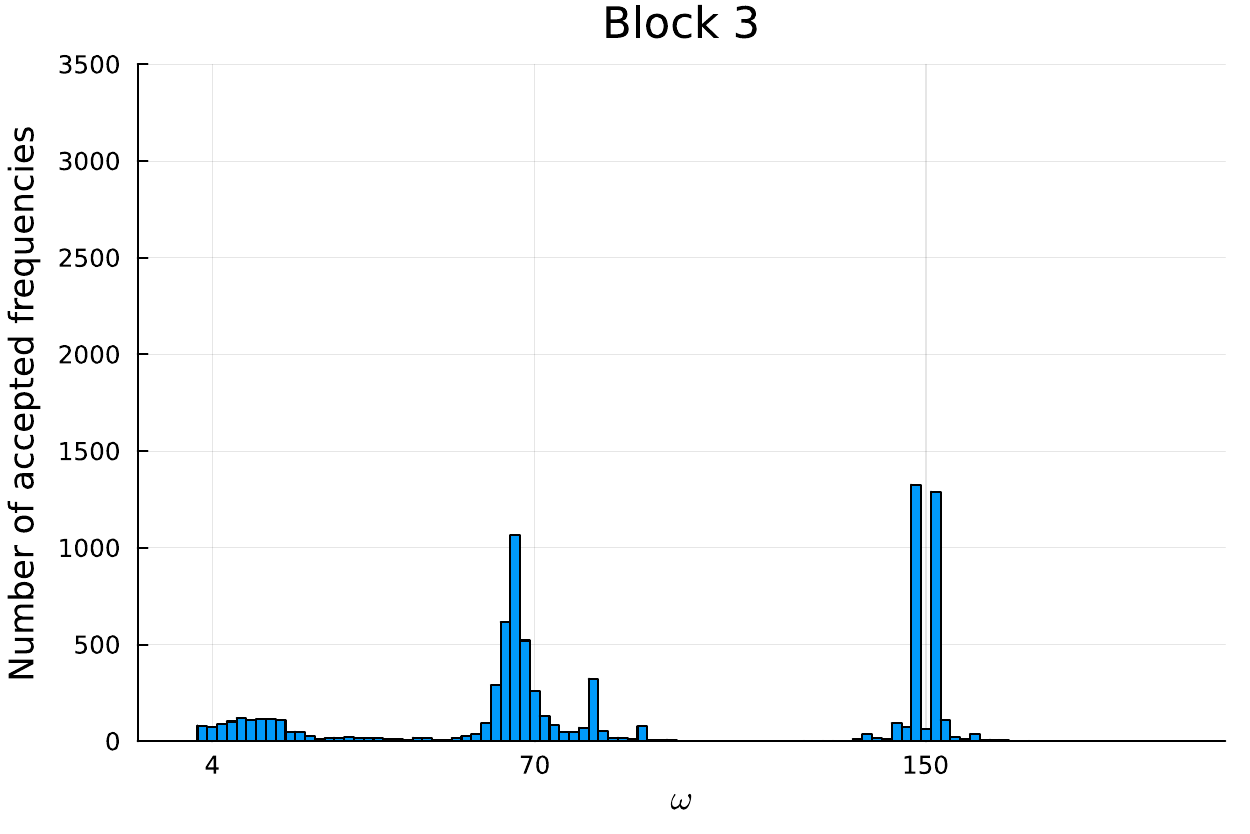}
    \caption{On the top row, network predictions (red dashed) versus the true target function (black solid) for the first three blocks of training (left to right), and on the bottom row, the accepted frequencies over all Markov Chains after a burn-in of $2000$ MCMC iterations.}
    \label{fig:multiscale_learning_progress}
\end{figure}

The training progression pictured in Figure~\ref{fig:multiscale_learning_progress} also shows that network depth combined with enforced residual learning between blocks provides a way for rFNNs using block-by-block training to efficiently learn small scale target function features. This is most apparent in the relative peak heights in the histograms between block $2$ and block $3$. Early in network training (blocks $1$ and $2$) the frequency $4$, and to a lesser extent $70$, are prioritized due to their large amplitudes relative to that associated with frequency $150$. However, at block 3, since we explicitly target the discrepancy $Q-z_{2}$, the amplitudes associated with the frequencies $4$ and $70$ have been reduced, and this facilitates a large relative increase in the number of frequencies sampled near $150$.

We additionally provide convergence results as a function of network complexity $WL$ in Figure~\ref{fig:multiscale_convergence}, where we observe much faster than the theoretical approximation rate over the first $10$ blocks of training.
\begin{figure}[!htb]
    \centering
    \includegraphics[width=0.4\textwidth]{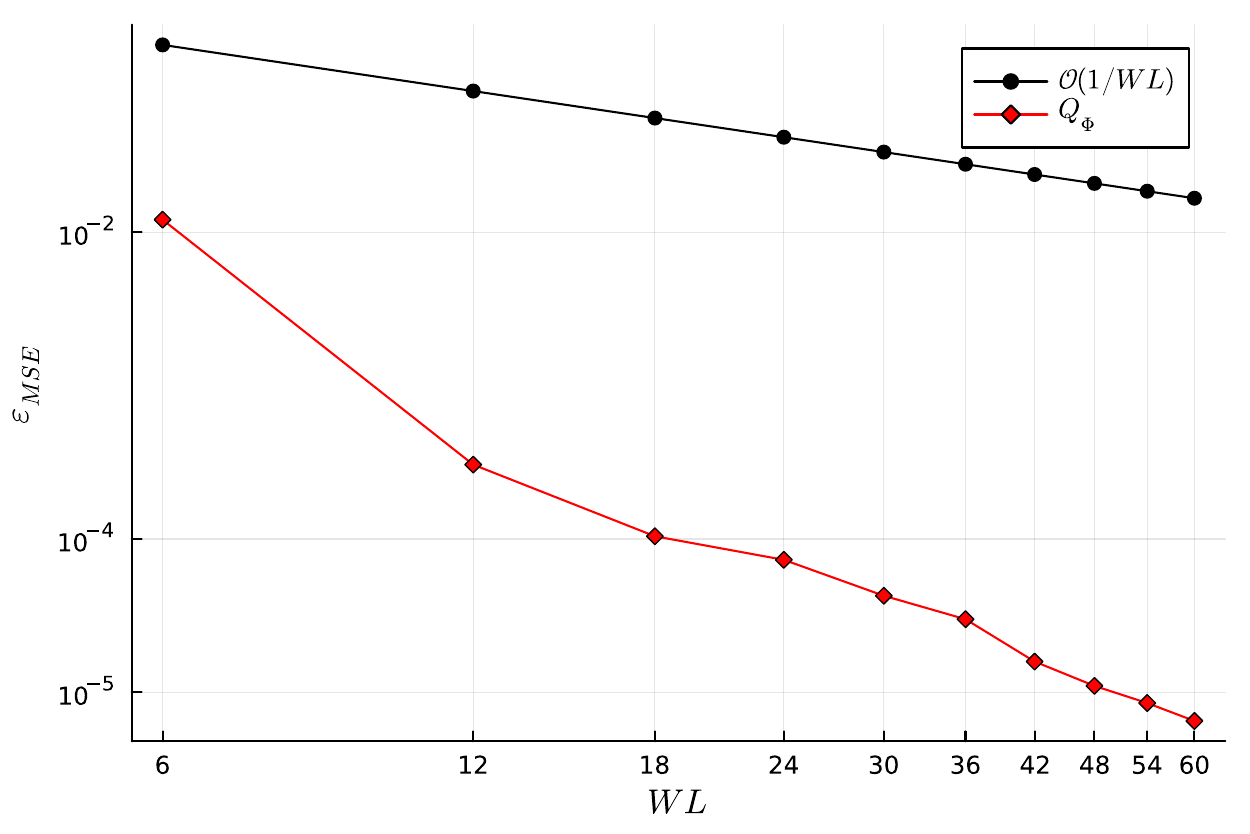}
    \caption{Mean squared error in the network predictions (red diamonds) and predicted convergence rate (black circles) as a function of network complexity $WL$ over the first 10 blocks of training.}
    \label{fig:multiscale_convergence}
\end{figure}

As a comparison with a global optimization-based training approach, in Figure~\ref{fig:multiscale_global}, we picture a network prediction from a Fourier neural network $Q^{global}_{\Phi}$ with architecture $(W,L) = (6,3)$ trained with global ADAM optimization for $30000$ epochs as well as the training loss curve. Note that the complexity of this network is the same as the block 3 prediction from the network trained with our block-by-block algorithm; see the rightmost plot in Figure~\ref{fig:multiscale_learning_progress}. The loss function for this global optimization procedure is given by
\begin{equation}\label{ADAM_loss}
    \mathcal{L}(\Phi) = \frac{1}{N}\sum_{n=1}^{N}|Q_{\Phi}^{global}(\theta^{(n)})-Q(\theta^{(n)})|^2 + \lambda|\bm{b}|^2,
\end{equation}
which is the mean squared error on the training set augmented by Tikhonov regularization on the amplitude parameters identical to that used in the block-by-block training algorithm, with Tikhonov regularization parameter $\lambda\geq0$. 
\begin{figure}[!htb]
    \centering
    \includegraphics[width=0.4\textwidth]{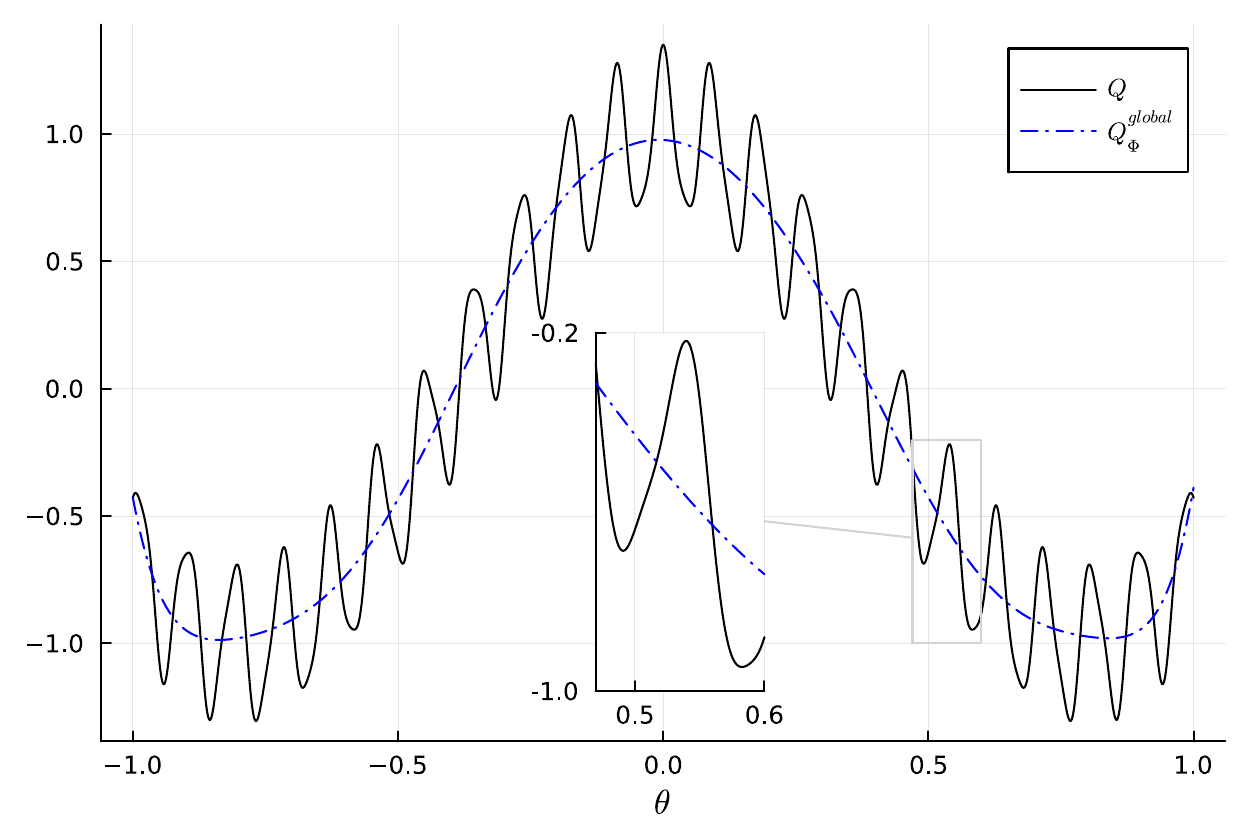}\includegraphics[width=0.4\textwidth]{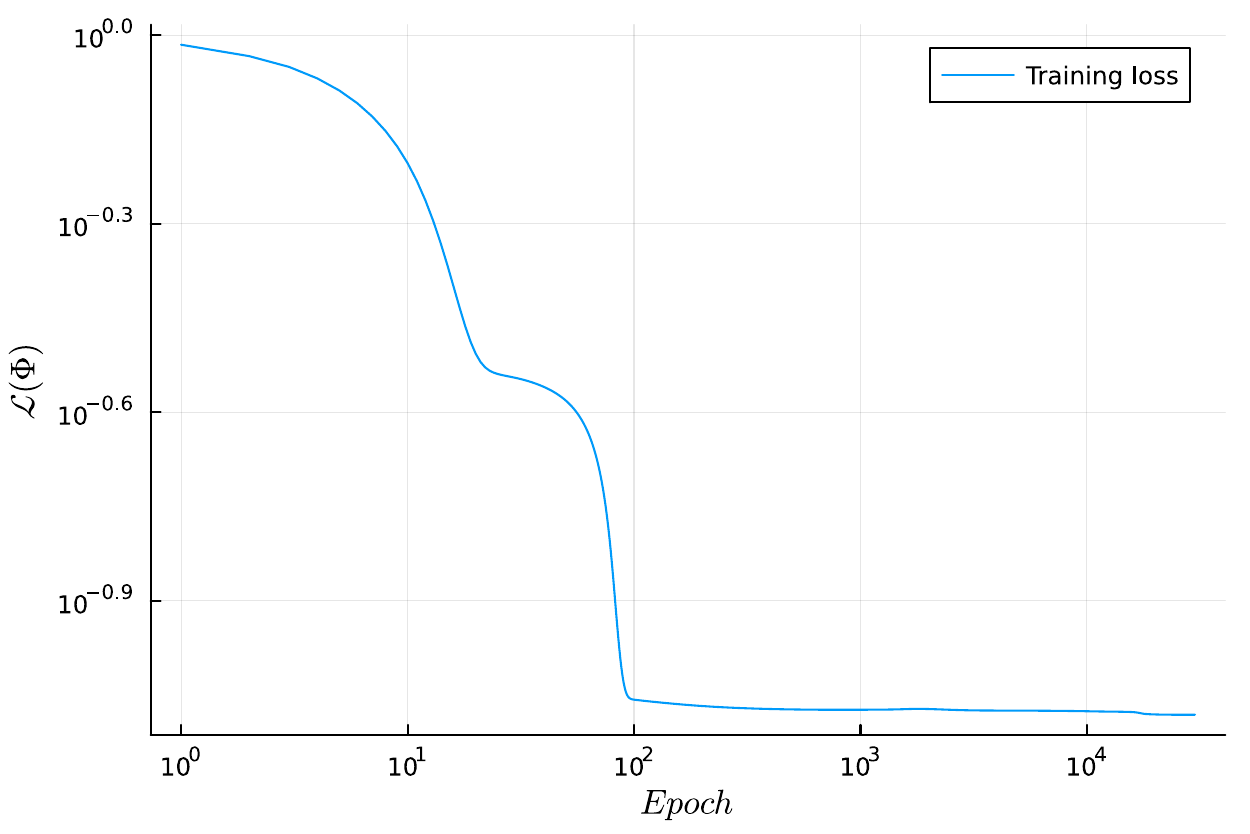}
    \caption{A Fourier neural network $Q^{global}_{\Phi}$ with architecture $(W,L)=(6,3)$ trained with global Adam optimization for $30000$ epochs (left) and the loss $\mathcal{L}(\Phi)$ on the training set a function of Epoch number (right) }
    \label{fig:multiscale_global}
\end{figure}
Despite this long training time for a one-dimensional problem, the resulting network prediction is considerably worse than the block $3$ prediction of the rFNN trained with our block-by-block algorithm. 
This result is consistent with the known connection between global gradient-based optimization (such as ADAM) and spectral bias. As evident from Figure 4, the network quickly learns the low frequency in just the first 100 epochs of training, but then completely stagnates for the remaining $\sim 30000$ epochs, and even after this long training time, does not capture either of the high-frequency target function features. We remark here as well that similar behavior is observed for deeper Fourier neural networks trained with ADAM. This observation provides evidence that the superior performance of rFNNs in approximating this oscillatory multiscale target function cannot be attributed to the use of a sinusoidal approximation basis. Without block-by-block training, the network still faces challenges in learning multiscale features with reasonable computational complexity.

Importantly, we do not claim that the Fourier neural network trained with ADAM is incapable of learning the target function given infinite training time and optimal hyperparameter choices, rather that given this multiscale target, the learning is remarkably slow compared to our block-by-block algorithm.

\subsection{A discontinuous target function}\label{sec:discont}

Consider the stairstep function pictured and defined in Figure~\ref{fig:stairstep}.
\begin{figure}[!h]
  \begin{minipage}{.5\textwidth}
    \begin{equation*}
    Q(\theta) =\begin{cases}
        0 & \theta\in [-1,-1/2)\\
        1/3 & \theta\in[-1/2,0)\\
        2/3 & \theta\in [0,1/2)\\
        1 & \theta\in [1/2,1].
    \end{cases}
\end{equation*}
  \end{minipage}%
  \begin{minipage}{.5\textwidth}
    \centering
    \includegraphics[width = 0.7\textwidth]{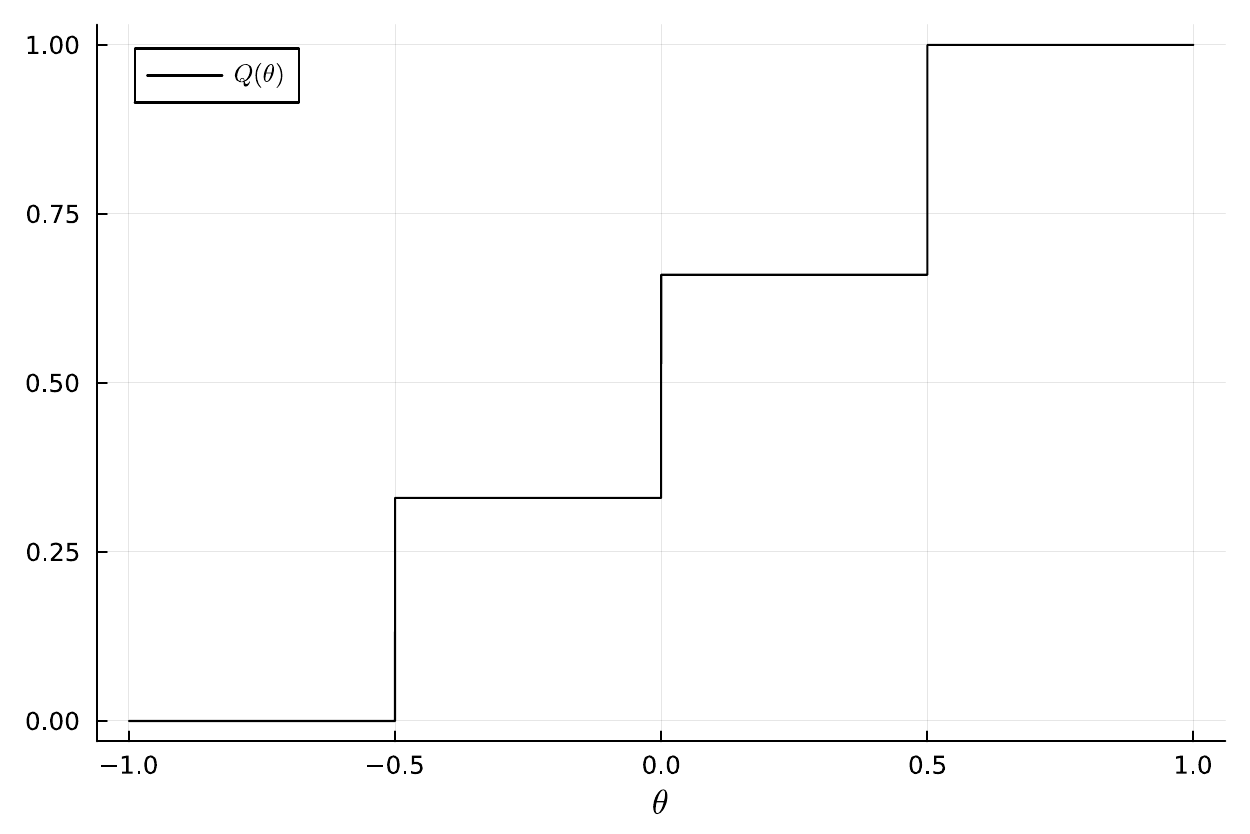}
  \end{minipage}
  \caption{A discontinuous target function $Q(\theta)$.}
    \label{fig:stairstep}
\end{figure}
\medskip

In this numerical example, our focus is on evaluating the performance and convergence rate of rFNNs in approximating discontinuous target functions using our block-by-block training algorithm. Moreover, we provide a direct comparison between our block-by-block training and the Metropolis algorithm used in \cite{deepFF_kammonen}. Unlike the multiscale and oscillatory target function discussed in Section~\ref{sec:multiscale}, there is no inherent advantage in using a sinusoidal approximation basis here. In fact, the opposite holds true. Fourier sum approximation of discontinuities formally requires an infinite number of terms, and finite Fourier sums exhibit Gibbs phenomena near discontinuities \cite{grafakos2008classical}.

We conduct an approximation of the stairstep function using rFNNs with two different training methods that we describe below.

\begin{itemize}
    \item {\bf Method 1} is the block-by-block training that we develop in this work. We call the resulting network $Q_{\Phi}^{(Method\:1)}$ and the output of each block $z_{\ell}^{(Method\:1)}$, $\ell=1,\dotsc L.$
    \item {\bf Method 2} is block-by-block training where the frequencies $\bm{\omega}_{\ell}'$ at each block $\ell>1$ are sampled once from a standard normal distribution and then never updated during training. This is exactly Algorithm 2 in \cite{deepFF_kammonen}. By not optimally sampling $\bm{\omega}_{\ell}'$, the expressiveness of the term $g_{\ell}'(z_{\ell-1};\bm{\omega}'_{\ell},\bm{b}'_{\ell})$ is limited, necessitating that $g_{\ell}(\theta,\bm{\omega}_{\ell};\bm{b}_{\ell})$ serve as the primary approximator of the target function at block $\ell>1$. We call the network resulting from this training $Q_{\Phi}^{(Method\:2)}$ and the output of each block $z_{\ell}^{(Method\:2)}$, $\ell=1,\dotsc, L$.
\end{itemize}
We use rFNNs of width $W=6$. In Figure~\ref{fig:stairstep_conv}, we plot the mean squared error over the first $10$ blocks of training. 
\begin{figure}[!htb]
    \centering\includegraphics[width = 0.4\textwidth]{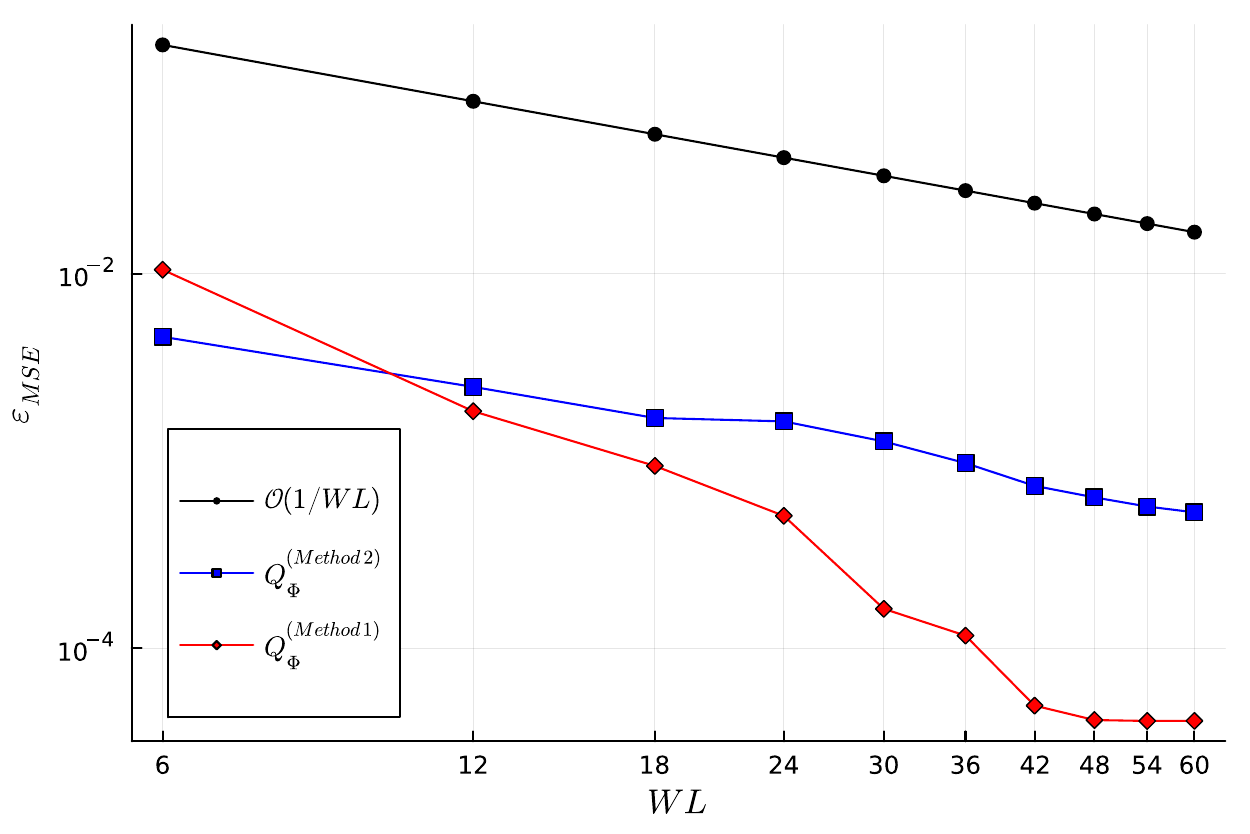}
    \caption{Mean squared error for networks approximating the stairstep function pictured in Figure~\ref{fig:stairstep} trained with Method 1 (red diamonds) and Method 2 (blue squares) after blocks $1$ through $10$ as a function of $WL$.}
    \label{fig:stairstep_conv}
\end{figure}
The red diamonds correspond to the approximation error in an rFNN trained with Method 1, while the blue squares represent a network trained with Method 2. The black circles denote the theoretical approximation rate. Notably, our block-by-block training surpasses the theoretical approximation rate and outperforms the network trained with Method 2; indeed, by block 10, the difference in error on the test set exceeds an order of magnitude. 
The error in the $Q_{\Phi}^{(Method\:1)}$ approximation appears to plateau in the later blocks of training, likely due to insufficient training data. With small tolerances, the error in the approximation concentrates near the discontinuities, suggesting that additional targeted training samples and network complexity may be required to further reduce the approximation error. The notable performance gap between Method 1 and Method 2 indicates that the optimal sampling of frequencies $\bm{\omega}_{\ell}'$ is crucial for effectively approximating discontinuities and functions with sharp features.

\begin{figure}[!htb]
    \centering
    \includegraphics[width=0.33\textwidth]{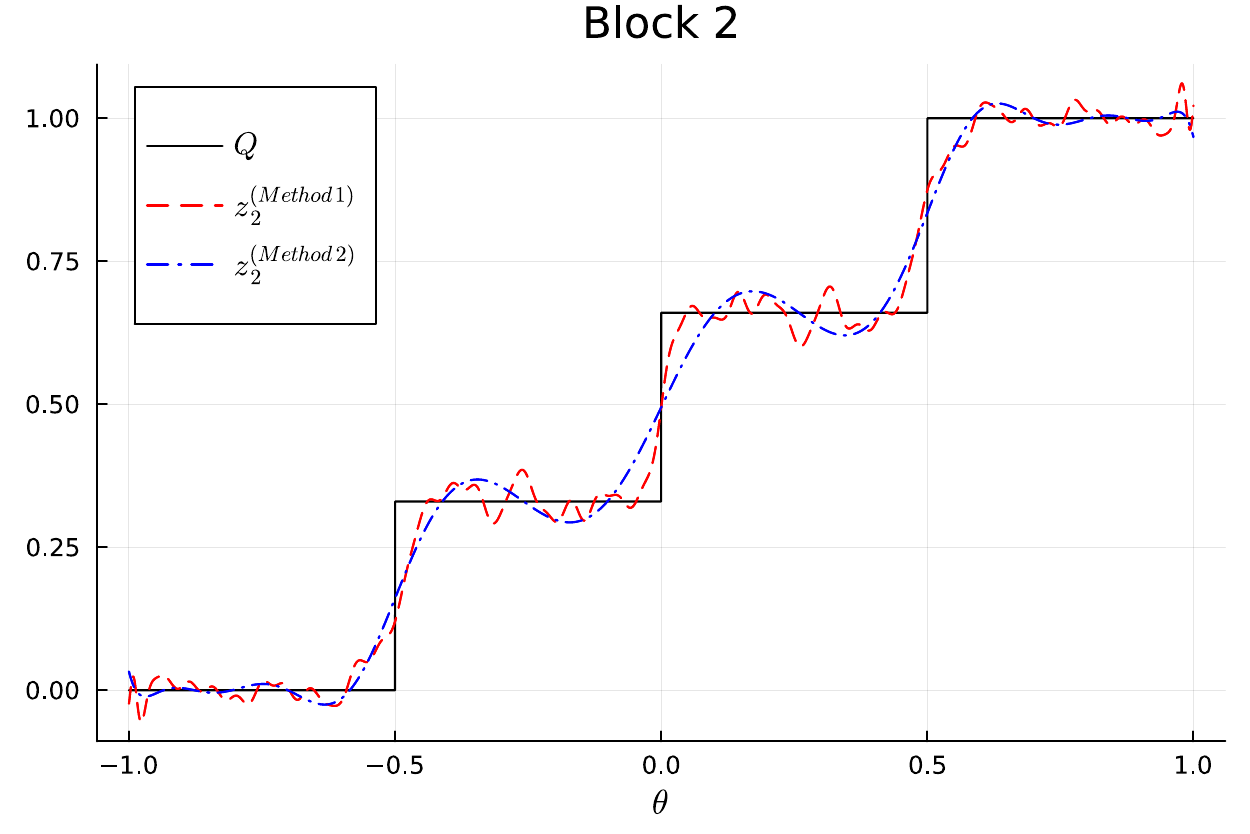}\includegraphics[width=0.33\textwidth]{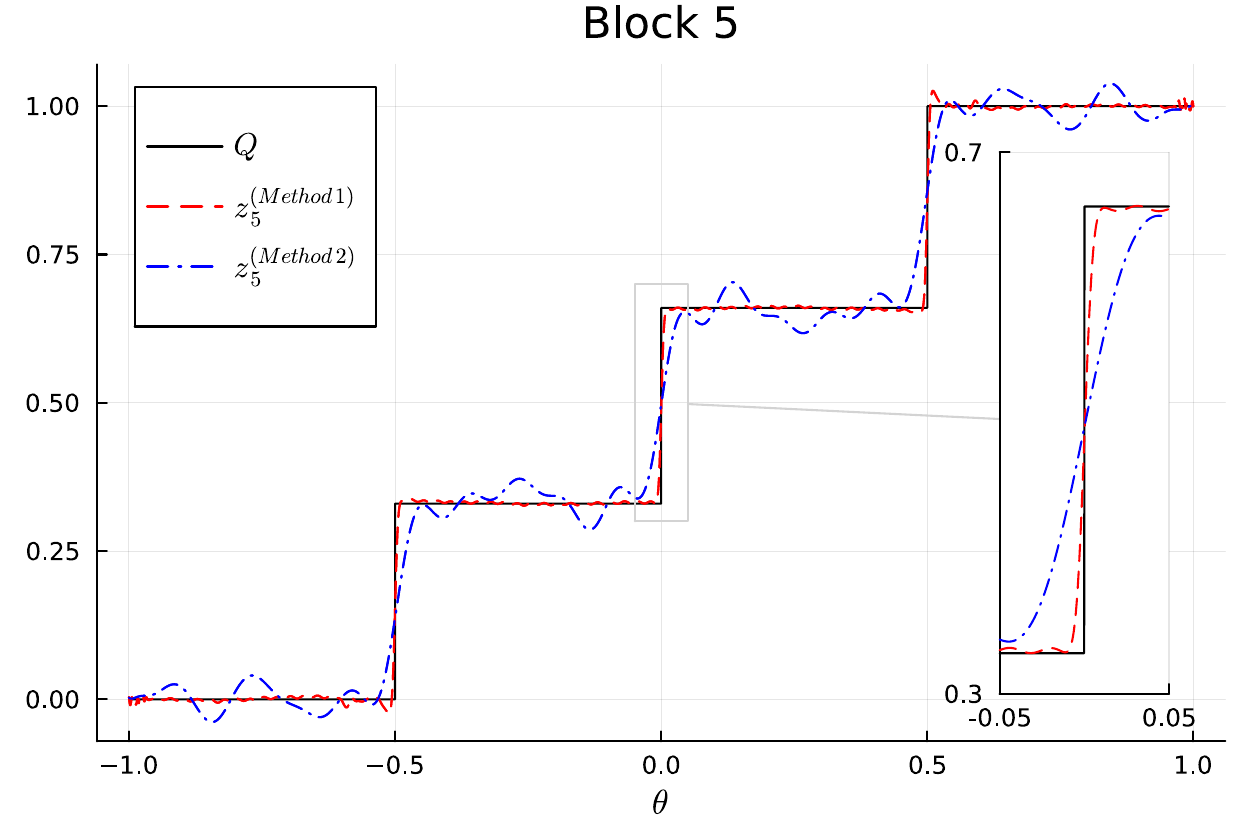}\includegraphics[width=0.33\textwidth]{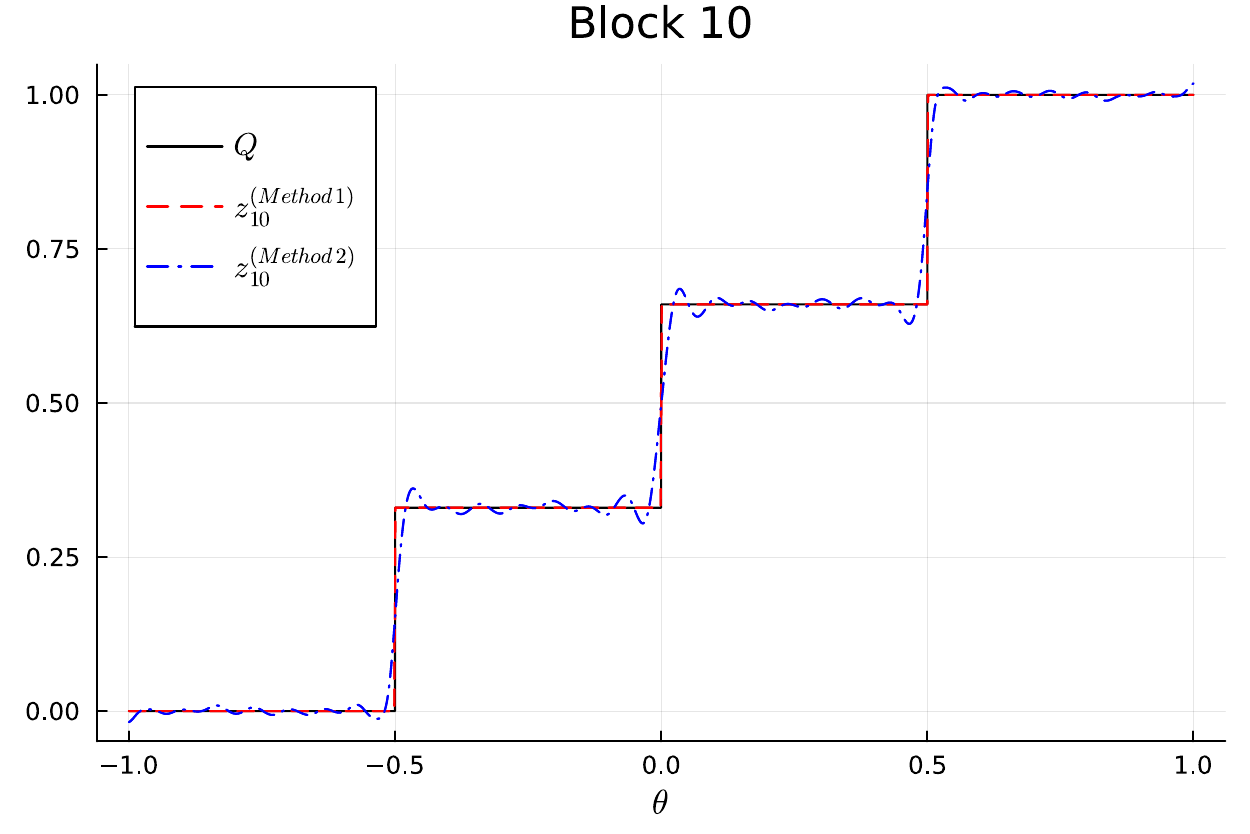}
    \caption{Network predictions for approximating the stairstep function pictured in Figure~\ref{fig:stairstep} after block 2 (left), block 5 (middle), and block 10 (right); results are pictured for a network trained with Method 1 (red dash) and Method 2 (blue dash-dot).}
    \label{fig:staistep_progression}
\end{figure}

In Figure~\ref{fig:staistep_progression}, we plot predictions from networks trained with both Method 1 (red dash) and Method 2 (blue dash-dot). We picture results after block 2 (left), block 5 (middle), and block 10 (right). By block 10, the network trained with Method 1 issues a prediction without Gibbs oscillations at the discontinuities. In contrast, the network trained with Method 2 offers a subpar prediction even at block 10, failing to effectively approximate the jump discontinuities. This qualitative analysis provides further evidence that the optimal sampling of the frequencies $\bm{\omega}_{\ell}'$ throughout training is essential to accurately approximating discontinuous target functions. 

We reiterate here that the frequencies $\bm{\omega}_{\ell}$ are trainable parameters associated with the network's standard Fourier modes, whereas the frequencies $\bm{\omega}_{\ell}'$ are associated with the network's basis functions which are compositions of Fourier modes. If the frequencies $\bm{\omega}_{\ell}'$ are not sampled optimally, as in Method 2 (Algorithm 2 from \cite{deepFF_kammonen}), then the brunt of the approximation has to be conducted by standard Fourier modes, and not surprisingly we observe Gibbs oscillations at the discontinuities. On the contrary, if the frequencies $\bm{\omega}_{\ell}'$ are sampled optimally, as in Method 1 (our block-by-block algorithm), then we are able to approximate the discontinuities sharply, from finite training data, and with relatively small network complexity. This supports our hypothesis that the compositional basis functions in rFNNs play an important role outside of simple identity mapping of the previous block's predictions, and that taking advantage of the additional expressivity requires that the frequencies $\bm{\omega}_{\ell}'$ be sampled optimally. These numerical results inspire potential future research concerning the approximation capabilities of basis functions which are nested compositions of Fourier modes. Overall, this section provides a direct comparison between our block-by-block training and the Metropolis algorithm used in \cite{deepFF_kammonen}. We observe benefit to our developed algorithm in the form of a faster approximation rate and improved ability to capture discontinuous features.

\subsection{A multidimensional target function}\label{sec:multidim}

Consider the three-dimensional regularized sine discontinuity given by 
\begin{equation}
    Q(\bm{\theta}) = e^{-\frac{|\bm{\theta}-\bm{c}|^2}{2}}\left(\int_{0}^{\frac{\theta^{(1)}-0.5}{0.1}}\frac{\sin(t)}{t}\:dt\right),
\end{equation}
where $\bm{c} = (0.5,0.5,0.5)$ and $\bm{\theta} = (\theta^{(1)},\theta^{(2)},\theta^{(3)})\in [0,1]^{3}$. This target functional is notable because it is nonlinear in all dimensions, and in dimension 1 has a frequency spectrum that decays slowly like $1/|\omega|$ for $\omega\in [0,10]$. We choose such a function intentionally to avoid showcasing a multidimensional example where the target function admits a notably simple and efficient Fourier representation. Here, our aim is to assess whether rFNNs trained with our block-by-block training algorithm can effectively approximate multidimensional functions and whether we achieve the predicted approximation rate. For this approximation task we use a network of width $W=4$. In Figure~\ref{fig:multidim}, we plot the mean squared error in this approximation (red diamonds) and the theoretical approximation rate (black circles) as a function of network complexity $WL$ over the first $10$ blocks of training. 
\begin{figure}[!htb]
    \centering
    \includegraphics[width=0.4\textwidth]{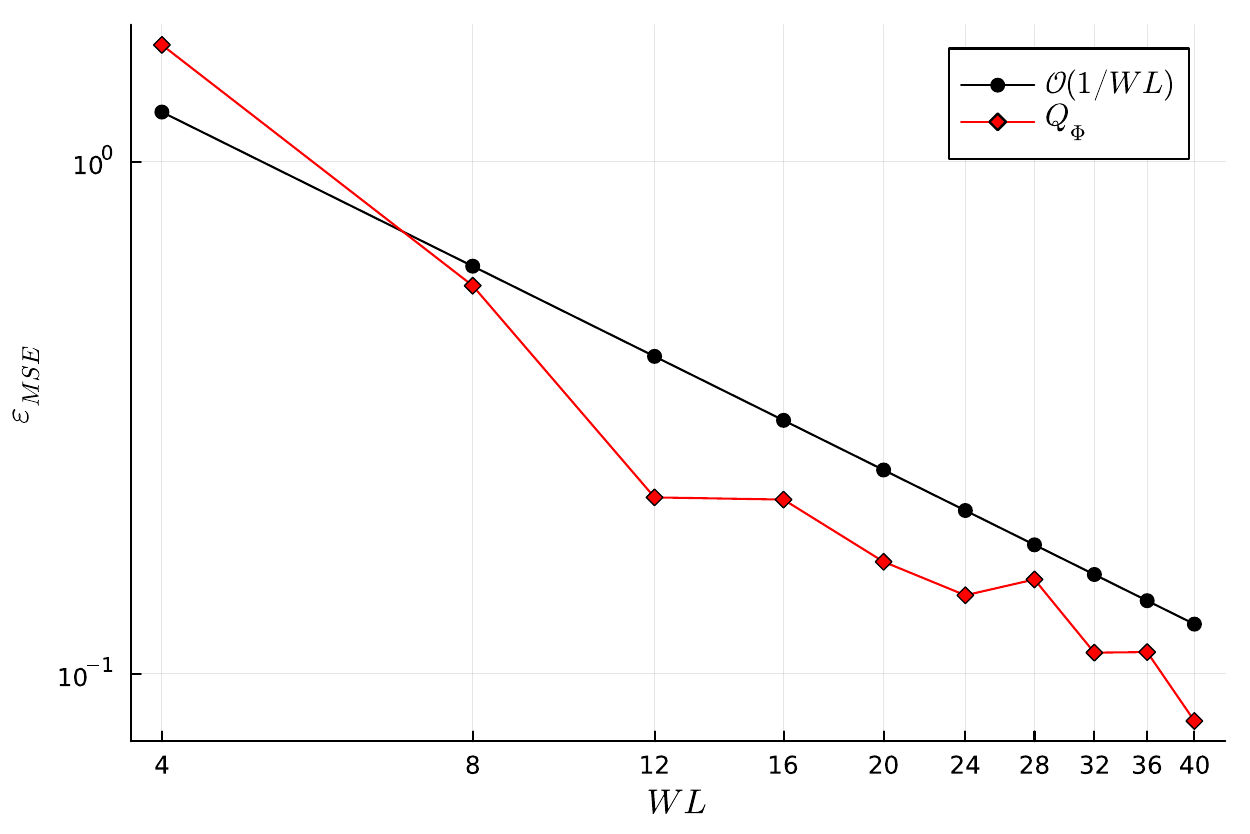}
    \caption{Mean squared error (red diamonds) and theoretical approximation rate (black circles) as a function of network complexity $WL$.}
    \label{fig:multidim}
\end{figure}
As seen in Figure~\ref{fig:multidim}, we recover the theoretical approximation rate. This is notable considering we use a network of very small width. Recall that the approximation error estimate in Theorem~\ref{thm:ff_main_result} theoretically requires that the network architecture satisfies $W=\mathcal{O}(L^2)$. However, in our case, we have $W<L$ (when $L \ge 5$), yet we still observe the theoretical approximation rate. 

We further emphasize that the ability to approximate a multidimensional target function with such small width is initially surprising. Typically, standard feedforward networks require that width grow with problem dimension. However, it is important to note that width in rFNNs differs from width in standard feedforward neural networks. Each neuron in a given block of an rFNN is associated with a $d$-dimensional frequency parameter, whereas in standard feedforward networks each neuron is associated with a $1$-dimensional weight parameter. Hence, although standard feedforward networks generally require increasing width with dimension, this is not necessarily a requirement for rFNNs.

\section{Conclusion}\label{sec:conclusion}

In this work, we developed a sampling-based training algorithm with error control for random Fourier neural networks. Unlike conventional neural network training algorithms, which consider a predefined network architecture, and then conduct global optimization over all network parameters simultaneously, our algorithm is iterative, training each block of the network in sequence via a Metropolis within Gibbs sampling procedure that seeks to sample a priori optimal distributions of frequency parameters at each block. Using this algorithm, the network architecture does not have to be specified ahead of time. Network blocks can be added and trained one at a time, calculating a chosen error metric after each addition. Training can stop once a desired tolerance is achieved. In this way, the algorithm provides a notion of error control with respect to network depth, and there is no need for global optimization of all network parameters simultaneously.

We evaluated our training algorithm on three numerical examples that highlighted different aspects of its behavior. In Section~\ref{sec:multiscale}, we showed that rFNNs trained with our block-by-block algorithm can approximate multiscale target function features with low network complexity, overcoming spectral bias. In Section~\ref{sec:discont}, we considered a discontinuous target function, and despite employing a sinusoidal approximation basis, we did not observe Gibbs oscillations. Furthermore, we emphasized the importance of optimal sampling of the frequencies $\bm{\omega}_{\ell}'$ in avoiding Gibbs phenomena and approximating target functions with sharp features. In Section~\ref{sec:multidim}, we showcased the capability to approximate multidimensional functions with rFNNs trained by our block-by-block algorithm. Additionally, over all numerical examples, we observed the only available theoretical approximation rate for rFNNs in terms of network complexity.

There are numerous directions for future research, including but not limited to exploring the scalability of the proposed algorithm and its deployment on high-performance computing resources, adapting the algorithm to networks with different activation functions and varying architectures, extending the algorithm to accommodate vector-valued target functions, optimizing the selection of training data, assessing its performance under sparser data conditions, and thoroughly characterizing and quantifying uncertainty in the algorithm. The extension to vector-valued target functions is theoretically straightforward  and the subject of ongoing work. Regarding uncertainty quantification, since the algorithm is MCMC-based, similar to Bayesian neural networks, obtaining distributions over network parameters comes at little additional training cost. Leveraging this capability to obtain reliable and embedded uncertainty estimates without incurring the usual high computational cost associated with training Bayesian networks represents a promising direction for future exploration.

\section*{Acknowledgements}

Sandia National Laboratories is a multi-mission laboratory managed and operated by National Technology \& Engineering Solutions of Sandia, LLC (NTESS), a wholly owned subsidiary of Honeywell International Inc., for the U.S. Department of Energy's National Nuclear Security Administration \newline (DOE/NNSA) under contract DE-NA0003525. This written work is authored by an employee of NTESS. The employee, not NTESS, owns the right, title and interest in and to the written work and is responsible for its contents. Any subjective views or opinions that might be expressed in the written work do not necessarily represent the views of the U.S. Government. The publisher acknowledges that the U.S. Government retains a non-exclusive, paid-up, irrevocable, world-wide license to publish or reproduce the published form of this written work or allow others to do so, for U.S. Government purposes. The DOE will provide public access to results of federally sponsored research in accordance with the DOE Public Access Plan.

\bibliographystyle{unsrt}
\bibliography{main}

\appendix

\section{Derivation of optimal frequency distributions}\label{appendix}

In this section, we derive the optimal frequency distributions for block $\ell>1$, by extending the arguments presented in \cite{1layerFF_kammonen}. This derivation proceeds by finding an upper bound on the block $\ell>1$ generalization error
\begin{equation*}
    \mathbb{E}_{\bm{\omega}_{\ell},\bm{\omega}_{\ell}'}[\min_{\bm{b}_{\ell},\bm{b}'_{\ell}}\{\mathbb{E}_{\theta}[|r_{\ell}(\theta, z_{\ell-1})-g_{\ell}(\theta)-g_{\ell}'(z_{\ell-1})|^2]+\lambda_{\ell}|\bm{b}_{\ell},\bm{b}_{\ell}'|^2 \}].
\end{equation*}
Recall that at block $\ell>1$, we assume that the target function has the form \\$r_{\ell}(\theta,z_{\ell-1}) = \bar{r}_{\ell}(\theta) + \bar{r}_{\ell}'(z_{\ell-1})$ for some unknown functions $\bar{r}_{\ell}$ and $\bar{r}'_{\ell}$. The first step in deriving the optimal frequency distributions is introducing the Fourier representations
\begin{equation}
    \bar{r}_{\ell}(\theta)= (2\pi)^{-d/2}\int_{\mathbb{R}^d}\hat{\bar{r}}_{\ell}(\omega)e^{i\omega\cdot \theta}\:d\omega,\qquad \bar{r}'_{\ell}(z_{\ell-1})= (2\pi)^{-1/2}\int_{\mathbb{R}}\hat{\bar{r}}'_{\ell}(\omega')e^{i\omega'z_{\ell-1}}\:d\omega',
\end{equation}
and their corresponding Monte Carlo estimators
\begin{equation*}
    g(\theta,\bm{\omega}_{\ell}) = \frac{1}{W}\sum_{j=1}^{W}\frac{\hat{g}(\omega_{\ell j})e^{i\omega_{\ell j}\cdot \theta}}
    {(2\pi)^{d/2}p_{\ell}(\omega_{\ell j})}, \qquad g'(z_{
    \ell-1
    },\bm{\omega}_{\ell}') = \frac{1}{W}\sum_{j=1}^{W}\frac{\hat{g}'(\omega'_{\ell j})e^{i\omega'_{\ell j}z_{\ell-1}}}{(2\pi)^{1/2}q_{\ell}(\omega'_{\ell j})}.
\end{equation*}
Recall here that $\omega_{\ell 1},\dotsc, \omega_{\ell W}$ are i.i.d. random variables with common marginal distribution $p_{\ell}:\omega\in \mathbb{R}^d\mapsto [0,\infty)$, and where $\omega_{\ell 1}',\dotsc,\omega_{\ell W}'$ are i.i.d. random variables with common marginal distribution $q_{\ell}:\omega'\in\mathbb{R}\mapsto [0,\infty)$. We further assume that $\bm{\omega}_{\ell}$ and $\bm{\omega}'_{\ell}$ are independent. Note that $g$ and $g'$ are unbiased estimators of $\bar{r}_{\ell}$ and $\bar{r}'_{\ell}$ respectively; that is 
\begin{equation*}
    \mathbb{E}_{\bm{\omega}_{\ell}}[g(\theta,\bm{\omega}_{\ell})] = \bar{r}_{\ell}(\theta),\quad \mathbb{E}_{\bm{\omega}_{\ell}'}[g'(z_{\ell-1},\bm{\omega}_{\ell}')] = \bar{r}'_{\ell}(z_{\ell-1}).
\end{equation*}
From here, the key insight lies in recognizing that the sum of the Monte Carlo estimators $g+g'$ shares the same structure as block $\ell>1$ of a random Fourier neural network with the particular amplitudes $\bm{\beta}=(\beta_{1},\dotsc, \beta_{W})$ and $\bm{\beta}' = (\beta'_{1},\dotsc, \beta_{W})$ given by
\begin{equation*}
\beta_{j} = \frac{\hat{\bar{r}}_{\ell}(\omega_{\ell j})}{W(2\pi)^{d/2}p_{\ell}(\omega_{\ell j})}, \qquad \beta_{j}' = \frac{\hat{\bar{r}}'_{\ell}(\omega_{\ell j}')}{W(2\pi)^{1/2}q_{\ell}(\omega_{\ell j}')}.
\end{equation*}
Therefore, to investigate approximation properties of block $\ell>1$, we analyze the specific version which corresponds to the sum of the Monte Carlo estimators $g+g'$. In particular, using that $\bm{\omega}_{\ell}$ and $\bm{\omega}_{\ell}'$ are independent and the definition of variance of a Monte Carlo estimator we calculate,
\begin{align*}
\mathbb{V}_{\bm{\omega}_{\ell},\bm{\omega}_{\ell}'}[g(\theta,\bm{\omega}_\ell)&+g'(z_{\ell-1},\bm{\omega}'_{\ell})] = \mathbb{V}_{\bm{\omega}_{\ell}}[g(\theta,\bm{\omega}_{\ell})]+\mathbb{V}_{\bm{\omega}_{\ell}'}[g'(z_{\ell-1},\bm{\omega}'_{\ell})]\\
&=\frac{1}{W}\mathbb{E}_{\omega}\left[\frac{|\hat{\bar{r}}_{\ell}(\omega)|^2}{(2\pi)^dp_{\ell}^2(\omega)}-\bar{r}^2_{\ell}(\theta)\right] + \frac{1}{W}\mathbb{E}_{\omega'}\left[\frac{|\hat{\bar{r}}'_{\ell}(\omega')|^2}{(2\pi)q_{\ell}^2(\omega')}-\bar{r}'^2_{\ell}(z_{\ell-1})\right].
\end{align*}
Now using this expression for the variance of the Monte Carlo estimators we can derive the following upper bound on the block $\ell>1$ generalization error

{\small
\begin{align*}
    \mathbb{E}_{\bm{\omega}_{\ell},\bm{\omega}_{\ell}'}[\min_{\bm{b}_{\ell}, \bm{b}_{\ell}'}\{\mathbb{E}_{\theta}[|r_{\ell}-g_{\ell}-&g_{\ell}'|^2]+\lambda_{\ell}|\bm{b}_{\ell},\bm{b}_{\ell}'|^2\}\leq \mathbb{E}_{\bm{\omega}_{\ell},\bm{\omega}_{\ell}'}[\mathbb{E}_{\theta}[|\bar{r}_{\ell}-g + \bar{r}'_{\ell}-g'|^2]+\lambda_{\ell}|\bm{\beta},\bm{\beta}'|^2 ]\\
    &\leq \frac{1}{W}\mathbb{E}_{\theta}\left[\mathbb{E}_{\omega}\left[\frac{|\hat{\bar{r}}_{\ell}(\omega)|^2}{(2\pi)^dp_{\ell}^2(\omega)}-\bar{r}^2_{\ell}(\theta)\right]\right] +\frac{\lambda_{\ell}}{W}\mathbb{E}_{\omega}\left[\frac{|\hat{\bar{r}}_{\ell}(\omega)|^2}{(2\pi)^dp_{\ell}^2(\omega)}\right]\\
    &+\frac{1}{W}\mathbb{E}_{\theta}\left[\mathbb{E}_{\omega'}\left[\frac{|\hat{\bar{r}}'_{\ell}(\omega')|^2}{(2\pi)q_{\ell}^2(\omega')}-\bar{r}'^2_{\ell}(z_{\ell-1})\right]\right] +\frac{\lambda_{\ell}}{W}\mathbb{E}_{\omega'}\left[\frac{|\hat{\bar{r}}'_{\ell}(\omega')|^2}{(2\pi)q_{\ell}^2(\omega')}\right]\\
    &\leq \frac{1+\lambda_{\ell}}{W}\left(\mathbb{E}_{\omega}\left[\frac{|\hat{\bar{r}}_{\ell}(\omega)|^2}{(2\pi)^{d}p_{\ell}^2(\omega)}\right] + \mathbb{E}_{\omega'}\left[\frac{|\hat{\bar{r}}'_{\ell}(\omega')|^2}{(2\pi)q_{\ell}^2(\omega')}\right]\right).
\end{align*}}
The final step is to show that this upper bound is minimized for the optimal distributions 
\begin{equation*}
    p_{\ell}^*(\omega) = \frac{|\hat{\bar{r}}_{\ell}(\omega)|}{||\hat{\bar{r}}_{\ell}||_{L^1(\mathbb{R}^d)}},\qquad q_{\ell}^*(\omega') =\frac{|\hat{\bar{r}}'_{\ell}(\omega)|}{||\hat{\bar{r}}'_{\ell}||_{L^1(\mathbb{R})}},
\end{equation*}
which we show in the following theorem. 

\begin{theorem}[minimizing probability densities]\label{thm:minimizing_densities}
    The probability densities 
\begin{equation*}
    p_{\ell}^*(\omega) = \frac{|\hat{\bar{r}}_{\ell}(\omega)|}{||\hat{\bar{r}}_{\ell}||_{L^1(\mathbb{R}^d)}},\qquad q_{\ell}^*(\omega') =\frac{|\hat{\bar{r}}'_{\ell}(\omega)|}{||\hat{\bar{r}}'_{\ell}||_{L^1(\mathbb{R})}},
\end{equation*}
are the minimizers of 
{\small
\begin{equation*}
    \min_{p_{\ell},q_{\ell}}\bigg\{\frac{1}{(2\pi)^d}\int_{\mathbb{R}^d}\frac{|\hat{\bar{r}}_{\ell}(\omega)|^2}{p_{\ell}(\omega)}\:d\omega + \frac{1}{2\pi}\int_{\mathbb{R}}\frac{|\hat{\bar{r}}'_{\ell}(\omega')|^2}{q_{\ell}(\omega')}\:d\omega'\: : \int_{\mathbb{R}^d}p_{\ell}(\omega)\:d\omega = 1, \int_{R}q_{\ell}(\omega')\:d\omega' = 1 \bigg\}.
\end{equation*}}
\end{theorem}
\begin{proof}
We conduct the change of variables 
\begin{equation*}
    p_{\ell}(\omega) = \frac{\bar{p}_{\ell}(\omega)}{\int_{\mathbb{R}^d}\bar{p}_{\ell}(\omega)\:d\omega},\qquad q_{\ell}(\omega') = \frac{\bar{q}_{\ell}(\omega')}{\int_{\mathbb{R}}\bar{q}_{\ell}(\omega')\:d\omega'}.
\end{equation*}
This implies that $\int_{\mathbb{R}^d}p_{\ell}(\omega)\:d\omega=1$ and $\int_{\mathbb{R}}q(\omega')\:d\omega'=1$ for any $\bar{p}:\mathbb{R}^d\mapsto [0,\infty)$ and $\bar{q}:\mathbb{R}\mapsto [0,\infty)$. Now for any $v:\mathbb{R}^d\mapsto\mathbb{R}$, $u:\mathbb{R}\mapsto \mathbb{R}$, and $\varepsilon>0$ let $f(\varepsilon) = f_{1}(\varepsilon) + f_{2}(\varepsilon)$ where $f_{1}$ and $f_2$ are defined as 
\begin{align*}
    f_1(\varepsilon) &= \int_{\mathbb{R}^d}\frac{|\hat{\bar{r}}_{\ell}(\omega)|^2}{\bar{p}_{\ell}(\omega) + \varepsilon v(\omega)}\:d\omega\int_{\mathbb{R}^d}\bar{p}_{\ell}(\omega) + \varepsilon v(\omega)\:d\omega;\\
    f_2(\varepsilon) &= \int_{\mathbb{R}}\frac{|\hat{\bar{r}}'_{\ell}(\omega')|^2}{\bar{q}_{\ell}(\omega') + \varepsilon u(\omega')}\:d\omega'\int_{\mathbb{R}}\bar{q}_{\ell}(\omega') + \varepsilon u(\omega')\:d\omega'.
\end{align*}
Now we look for optimal distributions $\bar{p}_{\ell}$ and $\bar{q}_{\ell}$ by solving $\frac{d}{d\varepsilon}f(0) =0$. We calculate

\begin{align*}
    \frac{df}{d\varepsilon}(0) & = \frac{d f_1}{d\varepsilon}(0) + \frac{d f_2}{d\varepsilon}(0)\\
    &=\int_{\mathbb{R}^d}\left(c_2-c_1\frac{|\hat{\bar{r}}_{\ell}(\omega)|^2}{\bar{p}_{\ell}^2(\omega)}\right)v(\omega)\:d\omega + \int_{\mathbb{R}}\left(c_3-c_4\frac{|\hat{\bar{r}}'_{\ell}(\omega')|^2}{\bar{q}_{\ell}^2(\omega')}\right)u(\omega')\:d\omega',
\end{align*}
where
\begin{align*}
    c_1 &= \int_{\mathbb{R}^d}\bar{p}_{\ell}(\tilde{\omega})\:d\tilde{\omega},\qquad c_2=\int_{\mathbb{R}^d}\frac{|\hat{\bar{r}}_{\ell}(\tilde{\omega})|^2}{\bar{p}_{\ell}(\tilde{\omega})}\:d\tilde{\omega}\\
    c_3 &= \int_{\mathbb{R}}\bar{q}_{\ell}(\tilde{\omega}')\:d\tilde{\omega}',\qquad c_4 =\int_{\mathbb{R}}\frac{|\hat{\bar{r}}'_{\ell}(\tilde{\omega}')|^2}{\bar{q}_{\ell}(\tilde{\omega}')}\:d\tilde{\omega}'.
\end{align*}
Thus we find $\bar{p}_{\ell}(\omega) = \left(\frac{c_1}{c_2}\right)^{1/2}|\hat{\bar{r}}_{\ell}(\omega)|$ and $\bar{q}_{\ell}(\omega') = \left(\frac{c_3}{c_4}\right)^{1/2}|\hat{\bar{r}}'_{\ell}(\omega')|$, which then implies that the minimizing densities are given by
\begin{equation}
    p_{\ell}^*(\omega) = \frac{|\hat{\bar{r}}_{\ell}(\omega)|}{||\hat{\bar{r}}_{\ell}||_{L^1(\mathbb{R}^d)}},\qquad q_{\ell}^*(\omega') =\frac{|\hat{\bar{r}}'_{\ell}(\omega)|}{||\hat{\bar{r}}'_{\ell}||_{L^1(\mathbb{R})}},
\end{equation}
\end{proof}

\section{Hyperparameters for numerical examples}\label{app:experimental_design}

Here we provide hyperparameter settings used to obtain the results for each of our numerical examples. 
\medskip

\noindent{\bf Section~\ref{sec:multiscale} A multiscale target function.} In this example, we compared a network trained with our block-by-block algorithm and a network trained with the ADAM optimization algorithm.  Hyperparameter settings for block-by-block training are pictured in Table~\ref{tab:multiscale_block_by_block} and hyperparameter settings for the network trained with ADAM are pictured in Table~\ref{tab:multiscale_ADAM}.
\begin{table}[!htb]
    \scriptsize\centering
    \begin{tabular}{c|c|c|c|c|c|c|c|c|c|c}
         Training method & $N$ &$N_{test}$ & $W$ & $L$ & $M$ & $\gamma$ & $\gamma'$ & $\delta$ & $\delta'$ &$\lambda_1,\dotsc,\lambda_{L}$\\
        \hline
         block-by-block
         & $2000$ & $10000$ & $6$ & $10$ & $20000$ & $10$ & $10$ & $2.4^2$ & $2.4^2$ & $1e-4$
    \end{tabular}
    \caption{Section~\ref{sec:multiscale} hyperparameter settings for block-by-block training}
    \label{tab:multiscale_block_by_block}
\end{table}
\begin{table}[!htb]
    \scriptsize\centering
    \begin{tabular}{c|c|c|c|c|c|c|c|c}
         Training method & $N$ &$N_{test}$ & $W$ & $L$ & Epochs & learning rate & $\lambda$ & batch size\\
        \hline
        ADAM
         & $2000$ & $10000$ & $6$ & $10$ & $15000$ & $1e-3$ & $1e-4$ & $256$
    \end{tabular}
    \caption{Section~\ref{sec:multiscale} hyperparameter settings for ADAM based training}
    \label{tab:multiscale_ADAM}
\end{table}
\medskip

\noindent{\bf Section~\ref{sec:discont}: A discontinuous target function.} In this, example we compared two different training methods.
\begin{table}[!htb]
    \scriptsize\centering
    \begin{tabular}{c|c|c|c|c|c|c|c|c|c|c}
        Training method & $N$ &$N_{test}$ & $W$ & $L$ & $M$ & $\gamma$ & $\gamma'$ & $\delta$ & $\delta'$ &$\lambda_1,\dotsc,\lambda_{L}$\\
        \hline
        Method 1 & $1000$ & $10000$ & $6$ & $10$ & $5000$ & $10$ & $10$ & $2.4^2$ & $2.4^2$ & $1e-6$ \\
        Method 2 & $1000$ & $10000$ & $6$ & $10$ & $5000$ & $10$ & N/A& $2.4^2$ & N/A & $1e-6$ 
    \end{tabular}
    \caption{Section~\ref{sec:discont} hyperparameter settings}
    \label{tab:discont}
\end{table}
The first was our block-by-block training algorithm (Method 1) and the second was our block-by-block training algorithm, but where the frequencies $\bm{\omega}_{\ell}'$ were sampled just once from a normal distribution and then never updated during training (Method 2). The hyperparameter settings for both methods are included in Table~\ref{tab:discont}.

\medskip

\noindent{\bf Section~\ref{sec:multidim}: A multidimensional target function.} In this example, we approximated a multidimensional target function using a network trained with our block-by-block algorithm. Hyperparameter settings to reproduce these results are included in Table~\ref{tab:multidim}.

\begin{table}[!htb]
    \scriptsize\centering
    \begin{tabular}{c|c|c|c|c|c|c|c|c|c|c}
         Training method & $N$ &$N_{test}$ & $W$ & $L$ & $M$ & $\gamma$ & $\gamma'$ & $\delta$ & $\delta'$ &$\lambda_1,\dotsc,\lambda_{L}$\\
        \hline
        block-by-block & $20^3$ & $25^3$ & $4$ & $10$ & $5000$ & $20$ & $20$ & $2.4^2/3$ & $2.4^2$ & $1e-4$ 
    \end{tabular}
    \caption{Section~\ref{sec:multidim} hyperparameter settings}
    \label{tab:multidim}
\end{table}

\end{document}